\documentclass{article}
\PassOptionsToPackage{numbers, compress}{natbib}
\usepackage[main, final]{neurips_2026}

\usepackage[utf8]{inputenc}
\usepackage[T1]{fontenc}
\usepackage{microtype}
\usepackage{graphicx}
\usepackage{float}
\usepackage{subcaption}
\usepackage{booktabs} %
\usepackage{hyperref}
\usepackage{url}
\usepackage[ruled,linesnumbered]{algorithm2e}
\usepackage{multirow}
\usepackage{makecell}

\usepackage{caption}

\usepackage[table]{xcolor}

\usepackage{amsmath}
\usepackage{bm}
\usepackage{amssymb}
\usepackage{mathtools}
\usepackage{amsthm}
\usepackage{nicefrac}

\usepackage[capitalize,noabbrev]{cleveref}

\theoremstyle{plain}
\newtheorem{theorem}{Theorem}[section]

\newtheorem{lemma}[theorem]{Lemma}
\newtheorem{corollary}[theorem]{Corollary}

\theoremstyle{definition}

\newtheorem{property}[theorem]{Property}
\theoremstyle{remark}

\crefname{property}{Property}{Properties}

\title{Context Value Informed In-Context \\  Reinforcement Learning}  

\author{%
  Wenhao Zhang \\
  Shanghai Jiao Tong University\\
  \texttt{wenhao\_zhang@sjtu.edu.cn} \\
  \And
  Shao Zhang \\
  Shanghai Artificial Intelligence Laboratory\\
  \texttt{zhangshao@pjlab.org.cn} \\
  \And
  Xihuai Wang \\
  Shanghai Jiao Tong University\\
  \texttt{leoxhwang@sjtu.edu.cn} \\
  \And
  Yang Li\thanks{Corresponding authors.} \\
  Shanghai Jiao Tong University\\
  \texttt{yang.li.cs@sjtu.edu.cn} \\
  \And
  Ying Wen\footnotemark[2] \\
  Shanghai Jiao Tong University\\
  Shanghai Innovation Institute\\
  \texttt{ying.wen@sjtu.edu.cn} \\
}

\begin{document}

\setcounter{footnote}{1}
\maketitle

\begin{abstract}

In-context reinforcement learning (ICRL) has emerged as an effective paradigm for test time adaptation to unseen tasks without parameter updates.
However, existing ICRL methods can exhibit brittle and unstable adaptions, and the mechanisms underlying such adaptions remain poorly understood. 
We provide a policy-gradient view of ICRL and argue that relying on trajectory-level return feedback can lead to high-variance updates when test time rollout is limited.
Motivated by the actor--critic principle, we propose \textbf{Context Value Informed ICRL (CV-ICRL)}, which equips the in-context policy with an explicit \emph{context value}: the expected discounted return conditioned on the current state and accumulated interaction context.
CV-ICRL trains a value head and writes its prediction back into the context as a value token, enabling TD-style return targets for lower-variance policy updates.
Experiments on the Dark Room, Minigrid, and Procgen testbeds show that CV-ICRL substantially improves the stability of test time adaptation and achieves higher returns across tasks and environments. Code and data are available at \url{https://github.com/Bluixe/CV-ICRL}.

\end{abstract}

\section{Introduction}

In-context reinforcement learning (ICRL) aims to endow a single pretrained sequence model with the ability to adapt to previously unseen tasks at test time solely by conditioning on its interaction history, without parameter updates~\citep{brown2020language,chan2022data}. This paradigm is particularly appealing for cross-task generalization, where conventional reinforcement learning agents often overfit to their training environments and exhibit brittle transfer under distribution shift~\citep{finn2017model,cobbe2019quantifying}. A common approach operationalizes this idea by pretraining causal Transformers on trajectory corpora that emulate online RL, interleaving observations, actions, rewards, and episode boundaries, so as to enable context-dependent policy improvement and cross-episode learning at inference time~\citep{AD,ADeps}.

Despite the appeal of this paradigm, the adaptation process in ICRL remains largely implicit and exhibits unstable test-time behavior, sometimes degrading with additional interaction~\citep{2025survey,IDT}. Building on recent theory that interprets in-context learning as implicit gradient-based optimization over context-induced variables~\citep{von2023transformers,dherin2025learning}, we view ICRL test-time adaptation as a policy-gradient-like (PG-like)~\citep{sutton1999policy} credit assignment process, in which actions leading to higher realized returns are reinforced in subsequent predictions.

Under this PG-like view, a key limitation of prior ICRL training emerges. Existing approaches~\citep{AD,ADeps,IDT} mainly distill actions without explicitly supervising a value signal for test-time credit assignment. In the data-scarce test-time regime, where adaptation must happen from only a few rollouts, such return-based signals can be high-variance, contributing to brittle improvement. Inspired by how reinforcement learning addresses this via a learned critic~\citep{konda1999actor}, we seek to equip ICRL with an analogous value-based signal for guiding inference-time adaptation.

\begin{figure*}[t]
  \centering
  \includegraphics[width=0.9\textwidth]{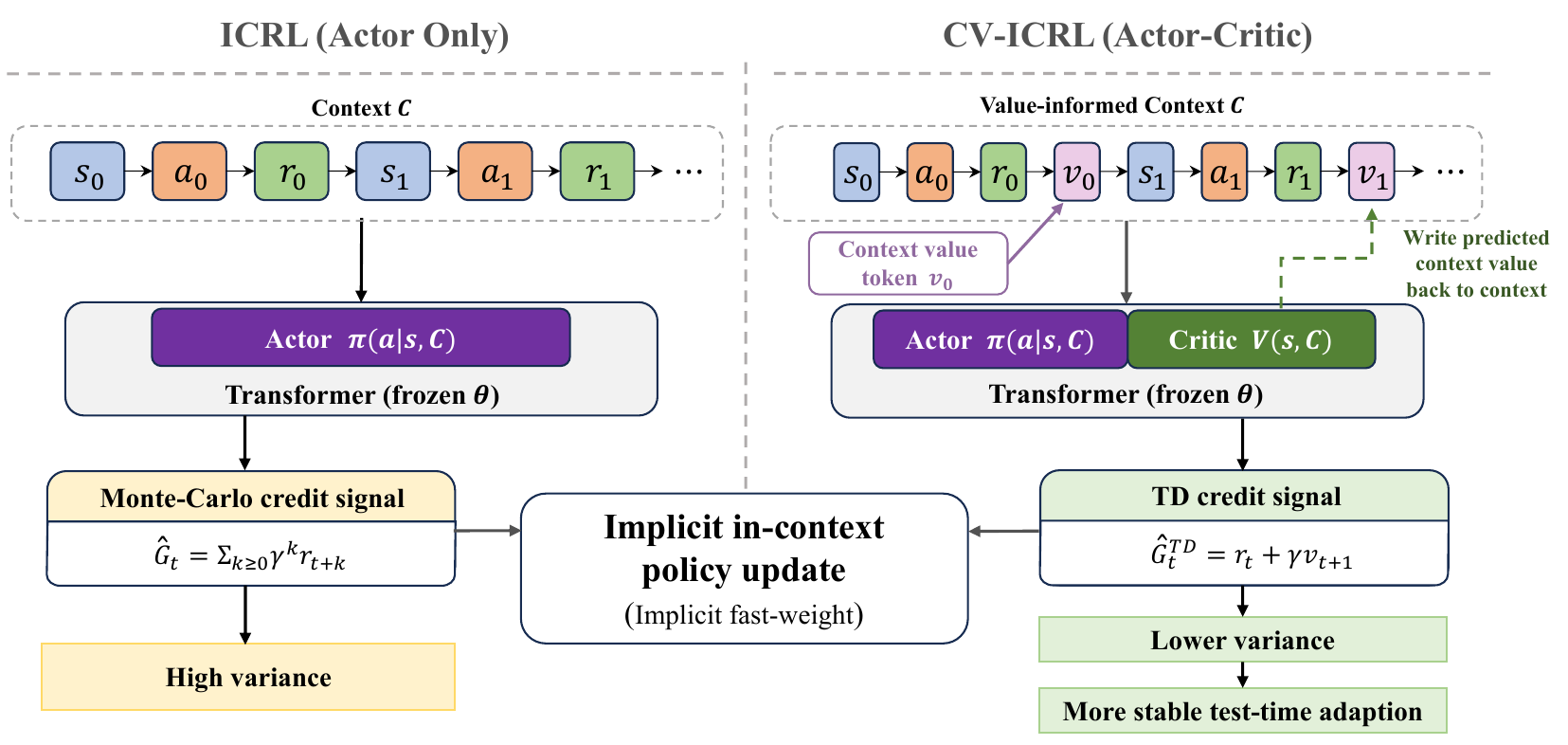}
  \caption{\textbf{From actor-only ICRL to context value informed (actor--critic-style) ICRL.} (\textit{Left}) The PG-like view models return-based credit assignment using the Monte-Carlo return $\widehat{G}_t=\sum_{k\ge0}\gamma^k r_{t+k}$. (\textit{Right}) CV-ICRL augments the same Transformer with a context critic that predicts a context value $v_{t+1}$ and writes it back into the context, enabling a TD-style one-step target $\widehat{G}_t^{\mathrm{TD}}=r_t+\gamma v_{t+1}$ for lower-variance test-time adaptation.}
  \label{fig:main_pg}
\end{figure*}
Based on the actor-critic principle, we propose \textbf{Context Value Informed ICRL (CV-ICRL)}, which augments an in-context policy with an explicit value signal, the \textbf{context value}, defined as the expected discounted return of the in-context policy conditioned on the current state and accumulated interaction context. As shown in~\Cref{fig:main_pg}, predicting the context value and writing it back into the context enables a bootstrapped one-step TD return~\citep{sutton1988learning} in place of the high-variance Monte-Carlo return. We show that, when using the true conditional value, this TD target is a Rao--Blackwellization of the Monte-Carlo return under mild measurability assumptions, providing a principled variance reduction for PG-style in-context updates. In practice, we implement CV-ICRL by distilling an online RL training run into a single Transformer following prior work~\citep{AD,IDT}, training the model to predict both each checkpoint's action distribution and a stable checkpoint-level proxy for the context value derived from the checkpoint's evaluated return.

In our experiments, we evaluate CV-ICRL on Dark Room, a more challenging Minigrid suite~\citep{MinigridMiniworld23}, and the visually complex Procgen benchmark~\citep{cobbe2019quantifying}. We empirically observe that prior ICRL method AD~\citep{AD} exhibits sharp drops in episode return at test time, consistent with the high-variance hypothesis. In contrast, CV-ICRL substantially improves adaptation stability and achieves higher average returns across all three testbeds. CV-ICRL also maintains robust adaptation under out-of-distribution shifts in dynamics or observations, and ablations confirm that both value prediction and value conditioning contribute to the gains.

In summary, our main contributions are given below:
\begin{itemize}
    \item We provide a policy-gradient-like (PG-like) viewpoint for in-context adaption of ICRL, which helps explain why prior actor-only distillation can be unstable at test time with only a few rollouts and provides an improvement direction for ICRL methods.

    \item We introduce \textbf{Context Value Informed ICRL (CV-ICRL)}, an actor--critic-style ICRL framework that predicts a long-horizon context value from the interaction history and appends it to the context to guide inference-time policy improvement with a lower-variance signal.
     \item We conduct extensive experiments on Dark Room, Minigrid, and Procgen, showing that CV-ICRL improves test time adaptation stability and average returns, with additional gains in out-of-distribution generalization.
\end{itemize}

\section{Related Work}

\paragraph{In-context Reinforcement Learning.}
In-context reinforcement learning (ICRL) uses causal Transformers to adapt from trajectory context, enabling few-shot cross-task generalization without gradient updates~\citep{beck2023survey, 2025survey}. A canonical example is the Decision Transformer (DT), which casts RL as conditional sequence modeling and predicts actions autoregressively from sequences of states, actions, and return-to-go (RTG) tokens~\citep{DT}.
Most existing ICRL approaches follow action distillation: Algorithm Distillation (AD) trains a Transformer on trajectories generated along an online RL run and supervises it primarily at the action level~\citep{AD}; later variants largely preserve this recipe~\citep{ADeps,AT,IDT}. Complementarily, Decision-Pretrained Transformer (DPT) studies supervised pretraining for ICRL from a posterior-sampling perspective~\citep{DPT}. These AD-like methods primarily supervise actions without explicitly supervising a value signal for test-time credit assignment. RTG conditioning makes goals more explicit~\citep{dai2024context,AT,IDT}, but test-time RTG updates still depend on sparse rewards. Some methods explore value heads~\citep{SICQL, jing2024opponent}; our focus is on writing predicted context values back into the cross-episode interaction history.
Reinformer also predicts a scalar return signal and conditions action selection on it, using expectile regression to estimate the maximum return-to-go in offline RL~\citep{zhuang2024reinformer}. CV-ICRL instead retains predicted context values across episodes to support adaptation to unseen tasks.
Our method also distills an online RL learning process into a sequence model, but introduces a learned context value that is inferred from and written into the context, enabling actor--critic-style TD credit assignment during inference for more stable test time adaptation.

\paragraph{ICRL as a PG-like Credit Assignment Process.}
Prior work shows that Transformers can adapt to prompt context at inference time with frozen parameters~\citep{brown2020language, dong2024survey}, and recent theory explains this as implicit optimization realized by the forward computation~\citep{wang2024context, dherin2025learning}.
Wang et al.~\citep{wang2025transformers_td} show that Transformers can learn TD methods for in-context policy evaluation, while Demircan et al.~\citep{demircan2024sparse_td} identify TD-error-like representations in LLMs solving RL tasks. These works study implicit TD-related computation; CV-ICRL makes a predicted context-value signal explicitly available through write-back.
Inspired by the policy-gradient view in reinforcement learning~\citep{williams1992simple, sutton1999policy}, we interpret ICRL as a PG-like credit assignment process within the context window: the in-context model induces a policy over actions, and the realized return from the interaction context plays the role of the credit signal. This perspective highlights a potential variance bottleneck under limited test-time rollouts.
In traditional RL, actor--critic methods mitigate this issue by learning a value baseline and using bootstrapping for lower-variance credit assignment~\citep{konda1999actor}.
These insights motivate our use of value-informed, actor--critic-style credit assignment.

\section{Preliminaries and Formulations}

\subsection{In-context reinforcement learning}
ICRL aims to produce a single policy that can adapt at test time to previously unseen tasks from only a small amount of online interaction, without updating model parameters.
Accordingly, an ICRL agent conditions its action not only on the current state but also on a context that summarizes its interaction history.
We formalize the testing environment as an MDP
$\mathcal{M}=(\mathcal{S},\mathcal{A},P,R,\gamma,\rho_0)$,
whose transition and reward functions are initially unknown to the agent.

At time $t$, the agent observes the current state $s_t$ together with a context that records past transitions and rewards, $C_t \triangleq \big((s_0,a_0,r_0),\ldots,(s_{t-1},a_{t-1},r_{t-1})\big)$ with $C_0=\emptyset$.
The agent then selects an action according to an in-context policy $\pi_\theta(a_t\mid s_t,C_t)$, where adaptation is realized through the evolution of $C_t$ while $\theta$ remains fixed.
After executing $a_t$ and observing $r_t$, the context is updated by concatenation:
\begin{equation}
\label{eq:context_update}
C_{t+1} = U(C_t,(s_t,a_t,r_t)).
\end{equation}

Given the unknown MDP $\mathcal{M}$, the objective of an in-context policy is to maximize the expected discounted return
\begin{equation}
\label{eq:icrl_single_task_objective}
J(\pi;\mathcal{M})
\triangleq
\mathbb{E}_{\pi_\theta,\mathcal{M}}
\left[\sum_{t=0}^{\infty}\gamma^t r_t\right].
\end{equation}

\subsection{A Policy-Gradient View of ICRL}
At test time, the ICRL agent interacts with the unknown task MDP $\mathcal{M}$ to maximize the objective in~\eqref{eq:icrl_single_task_objective}; the policy $\pi_\theta(a_t\mid s_t,C_t)$ is history-dependent through the growing context $C_t$. We adopt the augmented-state view $x_t \triangleq (s_t,C_t)$, under which $\pi_\theta(a_t\mid s_t,C_t) \equiv \pi_\theta(a_t\mid x_t)$ is a stationary Markov policy and the objective is preserved. This restores the Markov property required for the policy-gradient analysis below.

\paragraph{Fast weights induced by context.}
In many ICRL instantiations, $\pi_\theta$ is a pretrained causal Transformer with frozen weights. Although $\theta$ is fixed, the model computes a context-dependent internal representation that controls its action distribution. We make this explicit by introducing \emph{fast weights} $\phi_t \triangleq \Phi_\theta(C_t)$ such that $\pi_\theta(a_t\mid s_t,C_t) \equiv \pi(a_t\mid s_t,\phi_t)$, where $\phi_t$ can be understood as the Transformer's hidden state / cache summarizing the interaction history. As the context grows, $\phi_t$ changes according to an implicit update rule induced by the model computation:
\begin{equation}
\label{eq:fast_weight_update}
\phi_{t+1} = U_\theta(\phi_t, s_t,a_t,r_t).
\end{equation}
\citet{dherin2025learning} relate contextual-block computation to a query-dependent implicit rank-one update $\Delta W$ of an MLP weight $W$. This motivates our context-dependent fast-weight abstraction $\phi_t=\Phi_\theta(C_t)$, whose evolution is summarized by~\eqref{eq:fast_weight_update} while $\theta$ remains fixed.

\paragraph{A policy-gradient interpretation.}
The policy-gradient theorem on the augmented-MDP objective $J$, with $\phi$ as the inner-loop variables, gives
\begin{equation}
\label{eq:pg_phi}
\nabla_{\{\phi_t\}} J = \mathbb{E}\!\left[\sum_{t\ge 0} \nabla_{\phi_t}\log \pi(a_t\mid s_t,\phi_t)\, \widehat{G}_t \right],
\end{equation}
i.e., the ideal direction for improving expected return is a return-weighted score-function step. We interpret the architecture-determined update~\eqref{eq:fast_weight_update} as a return-aware specialization of this direction---a policy-gradient-style step on $\phi$:
\begin{equation}
\label{eq:pg_like_update}
\phi_{t+1} \approx \phi_t + \alpha\, \nabla_{\phi_t}\log \pi(a_t\mid s_t,\phi_t)\, \widehat{G}_t,
\end{equation}
under which actions associated with higher credited returns become more likely in subsequent predictions. This PG-like correspondence is a working interpretation rather than a formal equivalence; it helps interpret potential variance bottlenecks in return-based credit assignment and motivate the actor-critic remedy that follows.

\paragraph{Monte-Carlo credit and variance.}
Without providing an explicit value/baseline signal for credit assignment, a common choice for credit assignment is the trajectory-level Monte-Carlo return,
\begin{equation}
\label{eq:mc_return}
\widehat{G}_t \triangleq \sum_{k \ge 0} \gamma^k r_{t+k}.
\end{equation}
However, $\widehat{G}_t$ aggregates all downstream stochasticity into a single scalar and does not bootstrap from any estimate of future outcomes, which can yield high-variance credit assignment under short-horizon, data-scarce adaptation.
This motivates incorporating a value-informed signal that enables bootstrapping within the context.

\section{Context Value Informed In-Context RL} 

\begin{figure*}[t]
  \centering
  \includegraphics[width=0.85\textwidth]{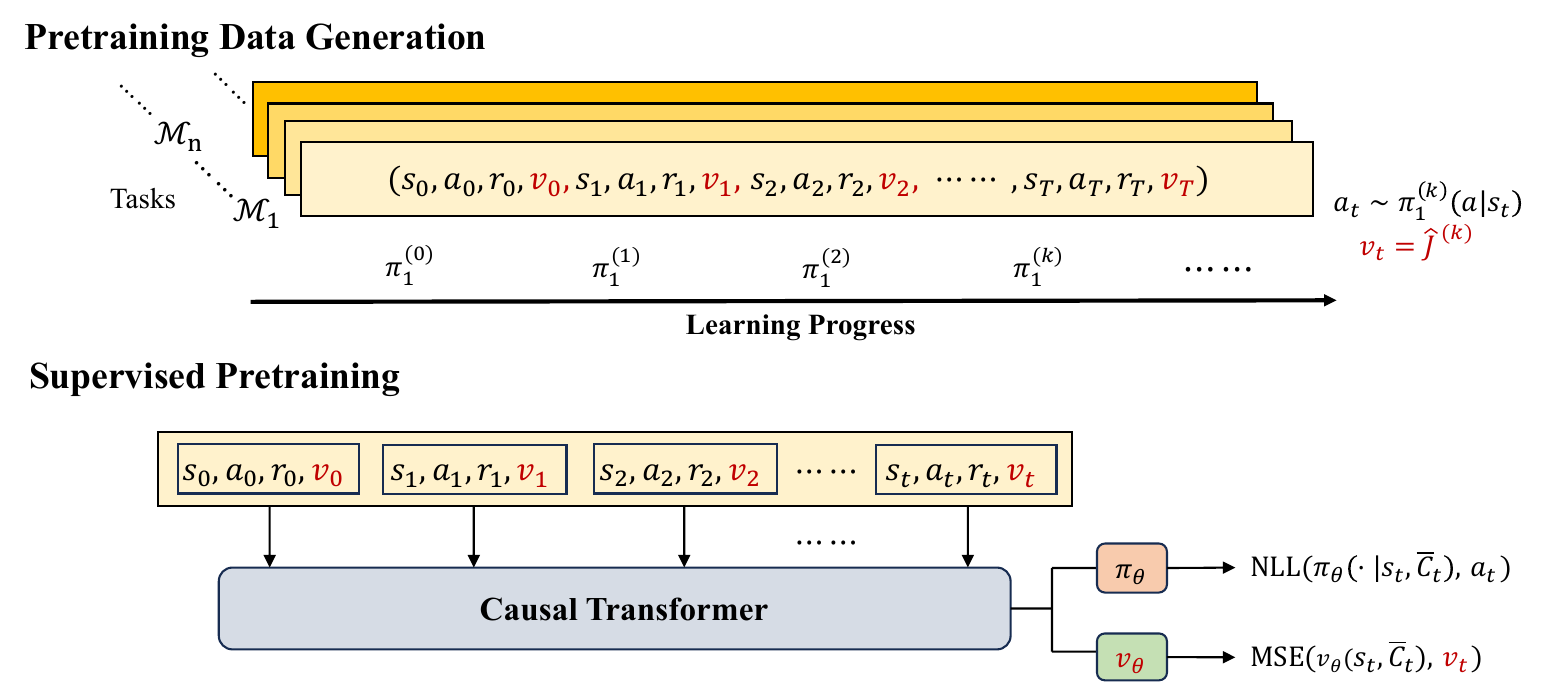}
  \caption{\textbf{Supervised pretraining of CV-ICRL from online RL learning progress.}
\textit{Top: Pretraining data generation.} We collect trajectories from training tasks along an online RL learning run, recording checkpoint policies $\{\pi^{(k)}\}$ and annotating each transition with a value token $v_t=\hat J^{(k)}$ (Monte-Carlo evaluated).
\textit{Bottom: Supervised pretraining.} A causal Transformer is pretrained to predict both the next action $\pi_\theta(\cdot\mid s_t,\bar C_t)$ and the context value $v_\theta(s_t,\bar C_t)$ from value-augmented history $\bar C_t$ using a joint loss: action NLL and value MSE.}
  \label{fig:main_pretrain}
\end{figure*}

We propose an actor--critic-like ICRL framework, \textbf{Context Value Informed ICRL (CV-ICRL)}, to reduce the variance of credit assignment. The key idea is to guide the in-context policy with an explicit value signal that reflects the agent's \emph{expected performance given the currently available context} $C$, enabling TD-style bootstrapping within the context. We formalize this signal as the \textbf{context value}.

\subsection{Context Value}
\label{sec:context_value}
Let $\pi(a\mid s,C)$ be an in-context policy and let the context evolve according to the update rule~\eqref{eq:context_update}. We define the \textbf{context value} as the standard discounted value on the augmented state $(s,C)$:
\begin{equation*}
V^{\pi}(s, C) \triangleq \mathbb{E}_{\pi,\,\mathcal{M}_{\mathrm{test}}}\!\left[\sum\nolimits_{k=0}^{\infty}\gamma^k r_{t+k} \,\middle|\, s_t=s,\, C_t=C\right],
\end{equation*}
where future actions follow $a_{t+k}\sim \pi(\cdot\mid s_{t+k},C_{t+k})$ and the context evolves by $C_{t+k+1}=U(C_{t+k},(s_{t+k},a_{t+k},r_{t+k}))$. Because $C$ can summarize interaction history across episodes, $V^\pi(s,C)$ measures the expected performance enabled by the information already inferred from context, rather than a state-only value. In our method we take $\pi=\pi_\theta$ (the pretrained Transformer policy) and write $V(s,C)\triangleq V^{\pi_\theta}(s,C)$ for brevity. Predicting and writing this signal back into the context is what enables the actor--critic-style, lower-variance test-time updates we describe next.

\subsection{CV-ICRL: Actor–Critic-Style In-Context Updates}
\textbf{CV-ICRL} augments a pretrained causal Transformer with an explicit \emph{context critic} that predicts the context value and, to make the critic information accessible to subsequent in-context computation, writes the value estimate back into the context. This \emph{augmented} context is 
$
\bar{C}_t \triangleq \big((s_0,a_0,r_0,v_0),\ldots,(s_{t-1},a_{t-1},r_{t-1},v_{t-1})\big)
$.

Given the current state $s_t$ and augmented context $\bar{C}_t$, the frozen Transformer outputs (i) an action distribution (actor) $\pi_\theta(\cdot \mid s_t, \bar{C}_t)$ and (ii) a scalar value estimate (critic) $v_t \triangleq v_\theta(s_t, \bar{C}_t)\approx V(s_t,\bar{C}_t)$. After sampling $a_t$ and observing $r_t$ and $s_{t+1}$, the context is updated as $\bar C_{t+1} \triangleq U(\bar{C}_t, (s_t, a_t, r_t, v_t))$. Context is retained across episode resets; once the FIFO window is full, the oldest transition is removed. \Cref{alg:cv_icrl} (Appendix~\ref{sec:algorithm}) summarizes the resulting training and test-time procedure.

\paragraph{From return-only credit to actor--critic credit.}
Replacing the base context $C_t$ in~\eqref{eq:fast_weight_update} by the augmented $\bar C_t$, the same fast-weight evolution applies, and the policy-gradient-style interpretation~\eqref{eq:pg_like_update} carries over with $\phi_t \equiv \Phi_\theta(\bar C_t)$. Compared with the Monte-Carlo credit signal in this interpretation~(\Cref{eq:mc_return}), CV-ICRL equips the context with a learned context-value estimate, enabling a bootstrapped one-step return target
\begin{equation}
\label{eq:cv_td_target}
\widehat{G}_t^{\mathrm{TD}} \triangleq r_t + \gamma v_{t+1}.
\end{equation}
Plugging~\eqref{eq:cv_td_target} into~\eqref{eq:pg_like_update} yields an actor--critic-style update,
\begin{equation}
\label{eq:ac_update}
\phi_{t+1} \approx \phi_t + \alpha\, \nabla_{\phi_t}\log \pi(a_t\mid s_t,\phi_t)\, \widehat{G}_t^{\mathrm{TD}},
\end{equation}
where the value token makes $v_{t+1}$ explicitly available in the context, allowing in-context updates to condition on a bootstrapped credit signal. The following result formalizes this variance reduction as a Rao--Blackwellization of the Monte-Carlo return when using the true conditional value.

\begin{property}[Variance reduction for PG-style updates with a bootstrapped one-step TD return]
\label{prop:rb_var_reduction}
Let $G_t=\sum_{k\ge0}\gamma^k r_{t+k}$ be the Monte-Carlo return and let
$V_{t+1} \triangleq V(s_{t+1},C_{t+1})$ denote the (true) context value at the next step.
Define the transition summary $H_t \triangleq (C_t,s_t,a_t,r_t,s_{t+1})$.
Let $\psi_t=\nabla_{\phi_t}\log \pi(a_t\mid s_t,\phi_t)$ and assume $\psi_t$ is deterministic given $(C_t,s_t,a_t)$.
Define the Rao--Blackwellized (bootstrapped) return target $\tilde G_t \triangleq \mathbb{E}\!\left[G_t\mid H_t\right]$. If $V_{t+1}$ is measurable w.r.t.\ $H_t$ (equivalently, $C_{t+1}$ is determined by $H_t$), then $\tilde G_t = r_t + \gamma V_{t+1}$, and $\tilde G_t$ is a Rao--Blackwellization of $G_t$. Consequently,
\begin{align}
\label{eq:rb_var_main_prop_return}
\mathrm{Var}\!\left(G_t\right) &\;\ge\; \mathrm{Var}\!\left(\tilde G_t\right),\\
\label{eq:rb_var_score_prop_return}
\mathrm{Var}\!\left(\psi_t G_t \mid C_t,s_t,a_t\right)
&\;\ge\;
\mathrm{Var}\!\left(\psi_t \tilde G_t \mid C_t,s_t,a_t\right).
\end{align}
Replacing $V_{t+1}$ with its in-context estimate $v_{t+1}$ yields the practical one-step TD target
$\widehat{G}_t^{\mathrm{TD}} \triangleq r_t+\gamma v_{t+1}$ in the main text.
The appendix also gives an approximation-error condition under which the practical target retains the variance advantage.
The proof is provided in Appendix~\ref{app:rb_proof}.
\end{property}

\subsection{Supervised Pretraining from RL Learning Progress}

Following prior ICRL works~\citep{AD,IDT}, we obtain supervised targets by distilling online RL training progress into a single Transformer.
We run RL algorithm on training tasks and record policy checkpoints $\{\pi^{(k)}\}_{k=1}^{K}$ along the training progress.
For each checkpoint $\pi^{(k)}$, we sample trajectory segments and form context--state pairs $(C_t,s_t)$ from the same task.
We use the checkpoint action distribution as the supervision signal, treating $\pi^{(k)}(\cdot\mid s_t)$ as a practical proxy for the context-conditioned response after observing $C_t$.

To supervise the value token, we use a stable long-horizon performance proxy rather than step-wise state values. Concretely, for each checkpoint $\pi^{(k)}$ we estimate its expected discounted return by Monte-Carlo evaluation, $\hat J^{(k)} \triangleq \frac{1}{N}\sum_{i=1}^{N}\sum_{t=0}^{T-1}\gamma^t r^{(i)}_t$, and use $\hat J^{(k)}$ to supervise the value head for samples collected under $\pi^{(k)}$. This provides a checkpoint-level competence proxy aligned with the context--policy correspondence induced by distillation, preserving learning-stage differences while averaging within-checkpoint state variation. Our analysis in~\Cref{cor:value_error} gives a sufficient value-error condition for a tighter performance bound under its policy-identification model.

\paragraph{Supervised pretraining objectives.}
During pretraining, the augmented history is $\bar C_t=((s_i,a_i,r_i,\hat J^{(k_i)}))_{i<t}$, where $k_i$ identifies each transition's source checkpoint. Given a context-state-action pair $(\bar C_t,s_t,a_t)$ collected under checkpoint $\pi^{(k)}$ with evaluated-return estimate $\hat J^{(k)}$, the training loss is
\begin{equation*}
\label{eq:cv_icrl_train_loss}
\mathcal{L}
=
\mathbb{E}_{k,t}\Big[
\underbrace{\mathcal{L}_{\pi}\big(\pi_\theta(\cdot\mid s_t,\bar C_t), a_t\big)}_{\text{policy distillation}}
+
\lambda \underbrace{\big( v_\theta(s_t,\bar C_t) - \hat J^{(k)}\big)^2}_{\text{value regression}}
\Big],
\end{equation*}
where $\mathcal{L}_{\pi}$ is implemented as negative log-likelihood, and $\lambda$ balances the two parts of loss. The source PPO generates action labels without consuming the added value annotation, whereas the Transformer is trained on the augmented context. At test time, the same temporal input structure uses the Transformer's own previous value predictions and interaction history on unseen tasks.

We also analyze CV-ICRL from an imitation-learning perspective. Under the appendix's policy-identification model and value-error condition, conditioning on the predicted value token $v(s,C)$ can reduce the worst-case deviation from the reference behavior and yields a tighter performance bound. The formal statement and proof are provided in Appendix~\ref{Appendix:improved_performance_bound}.

\section{Experiments}

\subsection{Experimental Setup}
\paragraph{Environments and Tasks.}We evaluate our methods on three testbeds: Dark Room~\citep{AD}, Minigrid~\citep{MinigridMiniworld23}, and Procgen~\citep{cobbe2019quantifying}. Dark Room is a commonly used environment for ICRL algorithms. Minigrid presents a more challenging and diverse benchmark, with tasks whose underlying MDPs differ in both reward functions and observation spaces; the consistent symbolic observation format across task types lets us rigorously test the model's cross-task generalization. Procgen substantially raises the visual and dynamic complexity beyond gridworlds: each task type uses 64$\times$64 RGB observations, procedurally generated levels with diverse layouts and enemy placements, and a 15-dimensional discrete action space. We use three Procgen task types: Bigfish, Starpilot, and Miner.

For Dark Room that has only one task type, we test ICRL policies on 20 unseen tasks that did not appear in training datasets.
For Minigrid and Procgen, within each task type (e.g. LavaCrossingS9N3), we train baselines and CV-ICRL methods in 400 tasks and test on 20 unseen tasks. Additionally, we test 4 unseen Minigrid task types to further demonstrate the generalization capabilities of ICRL policies.
More details of unseen tasks and task types are provided in Appendix~\ref{Appendix:env}.

\paragraph{Preparations of Pretraining Datasets.}
We collect training datasets from the PPO algorithm training process~\citep{ppo}. For Dark Room, we train PPO in the same way as AD. For Minigrid, the PPO policy faces instability when training directly from scratch, so we adopt a pretrain--finetune strategy: we first pretrain the model and then finetune it on hard tasks. For Procgen, we similarly train PPO checkpoints per task type to generate the trajectory dataset. For Minigrid value annotation, evaluating 40 checkpoints in 20 environments for 500 steps each requires 400,000 transitions per task type. More details can be found in Appendix~\ref{Appendix:pretrain}.

\begin{table*}[!t]
\caption{Minigrid results over 6 task types. We report equally weighted task means $\pm$ population standard deviations over 20 unseen tasks per task type. AER and LER are higher-is-better; Instability Frequency (IF, \%) is lower-is-better. IDT's IF is omitted for BlockedUnlockPickup because it fails on the task. Best results are \textbf{bolded}; second-best results are \underline{underlined}.}
    \vspace{-10pt}
\label{table:Minigrid}
\begin{center}
\resizebox{0.9\textwidth}{!}{%
\begin{tabular}{cc|lll>{\columncolor{lightgray!30}}l}
\toprule
\bf{Task Type} & \bf{Metric} & \bf{AD} & \bf{AD-$\epsilon$} & \bf{IDT} & \bf{CV-ICRL}  \\

\midrule
\multirow{3}{*}{\makecell{SimpleCrossing\\S9N3}}   & AER & \underline{$0.929\pm0.067$} & $0.152\pm0.119$ & $0.218\pm0.181$ &  $\bm{0.945\pm0.015}$ \\
                                    & LER & \underline{$0.941\pm0.029$} & $0.396\pm0.252$ & $0.529\pm0.226$ &  $\bm{0.949\pm0.012}$ \\
                                    & IF ($\%$) & \underline{$3.428\pm8.615$} & $81.378\pm10.588$ & $70.875\pm15.810$ & $\bm{1.561\pm1.880}$ \\
\midrule
\multirow{3}{*}{\makecell{SimpleCrossing\\S11N5}}   & AER & \underline{$0.865\pm0.124$} & $0.138\pm0.107$ & $0.323\pm0.296$ &  $\bm{0.902\pm0.076}$ \\
                                    & LER & \underline{$0.886\pm0.142$} & $0.329\pm0.242$ & $0.587\pm0.331$ &  $\bm{0.921\pm0.066}$ \\
                                    & IF ($\%$) & \underline{$15.754\pm15.686$} & $73.744\pm12.312$ & $67.190\pm26.947$ & $\bm{13.209\pm13.108}$ \\
\midrule
\multirow{3}{*}{\makecell{BlockedUnlock\\Pickup}}   & AER & \underline{$0.911\pm0.160$} & $0.219\pm0.370$ & $0.000\pm0.000$ &  $\bm{0.955\pm0.010}$ \\
                                    & LER & \underline{$0.911\pm0.209$} & $0.057\pm0.134$ & $0.000\pm0.000$ & $\bm{0.961\pm0.005}$ \\
                                    & IF ($\%$) & \underline{$2.493\pm7.373$} & $83.978\pm10.186$ & $-$ &  $\bm{1.825\pm2.732}$ \\
\midrule
\multirow{3}{*}{\makecell{Unlock}}   & AER & \underline{$0.940\pm0.025$} & $0.042\pm0.026$ & $0.258\pm0.157$ & $\bm{0.966\pm0.008}$ \\
                                    & LER & \underline{$0.941\pm0.029$} & $0.552\pm0.255$ & $0.492\pm0.239$ &  $\bm{0.969\pm0.007}$ \\
                                    & IF ($\%$) & \underline{$4.885\pm4.890$} & $86.588\pm10.775$ & $72.213\pm11.375$ &  $\bm{0.217\pm0.172}$ \\
\midrule
\multirow{3}{*}{\makecell{LavaCrossing\\S9N3}}   & AER & \underline{$0.918\pm0.035$} & $0.902\pm0.059$ & $0.915\pm0.057$ &  $\bm{0.921\pm0.051}$ \\
                                    & LER & $0.933\pm0.051$ & $0.936\pm0.052$ & $\bm{0.945\pm0.020}$ &  \underline{$0.939\pm0.026$} \\
                                    & IF ($\%$) & $\bm{3.706\pm4.456}$ & $6.149\pm9.108$ & $5.590\pm8.810$ & \underline{$5.228\pm8.787$}\\
\midrule
\multirow{3}{*}{\makecell{LavaCrossing\\S9N2}}   & AER & $\underline{0.937\pm0.027}$ & $0.918\pm0.032$ & $0.898\pm0.097$ &  $\bm{0.943\pm0.015}$ \\
                                    & LER & $\underline{0.954\pm0.012}$ & $0.946\pm0.034$ & $0.927\pm0.060$ &  $\bm{0.954\pm0.011}$ \\
                                    & IF ($\%$) & $\underline{2.101\pm2.501}$ & $3.524\pm3.735$ & $8.213\pm12.041$ &  $\bm{1.557\pm1.289}$ \\
\bottomrule
\end{tabular}}
\end{center}
    \vspace{-6pt}
\end{table*}

\begin{figure*}[t!]
\begin{center}
\includegraphics[width=0.7\textwidth]{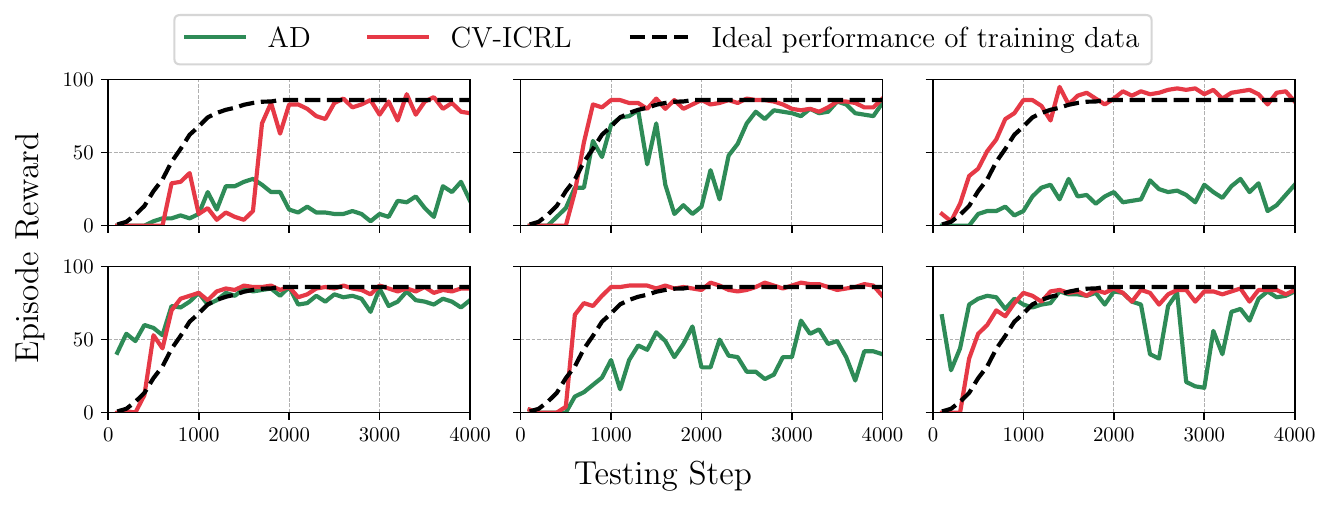}
\end{center}
\vspace{-10pt}
\caption{Across 6 different Dark Room tasks, AD struggles to maintain performance stability, whereas CV-ICRL substantially improves stability.}
\label{fig:case_study}
\vspace{-10pt}
\end{figure*}

\begin{figure*}[t]
  \centering
  \includegraphics[width=0.95\textwidth]{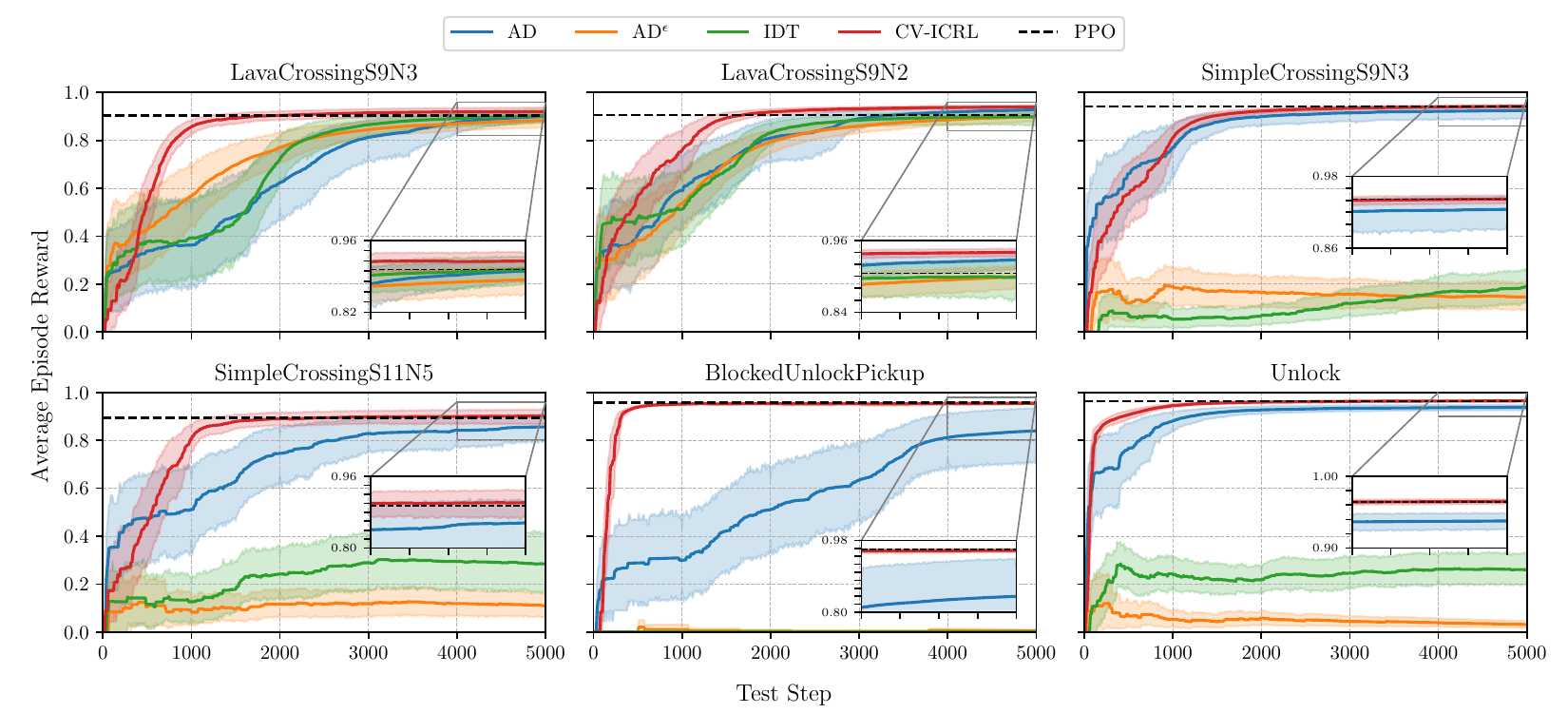}
  \vspace{-5pt}
  \caption{Test time AER in Minigrid. Curves show the mean over 20 unseen tasks with $95\%$ confidence intervals; the black dashed line is a single-task PPO reference trained directly on each evaluation task.}
  \label{fig:main_result}
  \vspace{-10pt}
\end{figure*}

\begin{figure*}[t]
  \centering
  \includegraphics[width=0.95\textwidth]{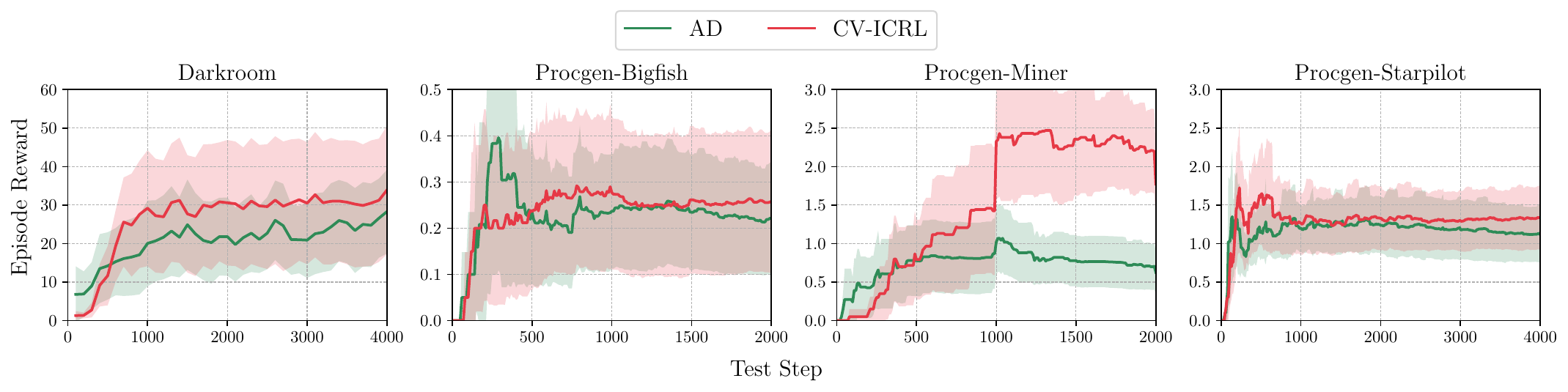}
  \vspace{-5pt}
  \caption{Test time episode reward in Dark Room and three Procgen task types. Curves show the mean over 20 unseen tasks/levels with $95\%$ confidence intervals.}
  \label{fig:darkroom_procgen}
  \vspace{-10pt}
\end{figure*}

\paragraph{Baselines.}
We use AD~\citep{AD}, AD-$\epsilon$~\citep{ADeps}, and IDT~\citep{IDT} as baselines. We additionally report DPT~\citep{DPT} results in Appendix~\ref{appendix:dpt_baseline} as a complementary ICRL baseline. We also compare with cross-episode adaptations of Reinformer and IC-CQL on four Minigrid task types; their protocols are provided in Appendix~\ref{appendix:additional_baselines}.
We implement AD, AD-$\epsilon$, IDT, and CV-ICRL on a GPT-2~\citep{gpt2} based backbone, ensuring comparable parameter scales and closely matched architectural hyperparameters. The AD/CV-ICRL comparison also matches data ordering and training budget. Added-baseline architectures and protocols are described separately. More details can be found in Appendix~\ref{appendix:imple}.

\paragraph{Metrics.} 
To evaluate the overall performance, we use \textbf{Average Episode Return} (AER) as our main metric. In addition, we also consider the \textbf{Last Episode Return} (LER), which is used in previous work~\citep{tarasovyes}, as it reflects the final performance, indicating the ultimate in-context learning outcomes.
In addition, to quantify the performance stability of ICRL policy, we propose a metric named \textbf{Instability Frequency}, which is calculated as the proportion of episodes in which the episode reward decreases by at least 5\% compared to the previous one. The Instability Frequency \( IF \) is given by $IF = \frac{1}{N} \sum_{i=1}^{N} \mathbb{I}(r_i \leq 0.95 \cdot r_{i-1})$,    
where \( N \) is the total number of episodes, \( r_i \) is the reward of the \( i \)-th episode, and \( \mathbb{I}(\cdot) \) is the indicator function.

\subsection{Main Results}

\paragraph{Case study: performance degradation following occasional low-return rollouts.}
To examine performance instability in data-scarce test time adaptation, we conduct a case study on Dark Room.
We observe that episode returns can occasionally drop sharply, which is consistent with sampling noise and the stochasticity of the policy/environment.
Under the PG-like interpretation, such low-return outliers can perturb return-based credit assignment.
Crucially, after a bad rollout occurs, AD~\citep{AD} often exhibits degraded performance over subsequent rollouts (and recovers slowly), consistent with sensitivity to noisy return feedback during adaptation (shown in \Cref{fig:case_study}).

\paragraph{The performance instability shown in case study is prevalent in many scenarios. CV-ICRL proves effective in alleviating this issue, leading to significant improvements in overall performance.}
Building on the initial observations from the case study, We then examine the prevalence of this behavior across a diverse set of tasks in Minigrid. Our experiments across 6 diverse tasks in the Minigrid environment confirm this prevalence. As presented in \Cref{table:Minigrid}, the results for baseline methods show frequent low-return episodes and performance degradation, reflected in a high Instability Frequency (IF) and lower AER. Against this backdrop, CV-ICRL delivers consistent gains in both stability and performance.
As reported in \Cref{table:Minigrid}, our method reduces IF relative to AD on five of the six task types, indicating less frequent performance degradation.
In \Cref{fig:main_result}, CV-ICRL reaches high returns earlier within the same interaction budget. LER gains are smaller on tasks where both methods approach the performance ceiling; AER and IF more clearly characterize differences along the adaptation path.
We observe the same trend on the Dark Room and Procgen benchmarks (\Cref{fig:darkroom_procgen}).
The temporal-variance analysis in Appendix~\ref{appendix:confidence_bands} shows that the wider aggregate confidence bands in some plots mainly reflect cross-task difficulty heterogeneity.
We further provide a small-scale continuous-control evaluation on MetaWorld in Appendix~\ref{appendix:metaworld}, where CV-ICRL improves over AD on both ML1-reach-v3 and ML1-push-v3.
\begin{figure*}[t]
  \centering
  \includegraphics[width=0.98\textwidth]{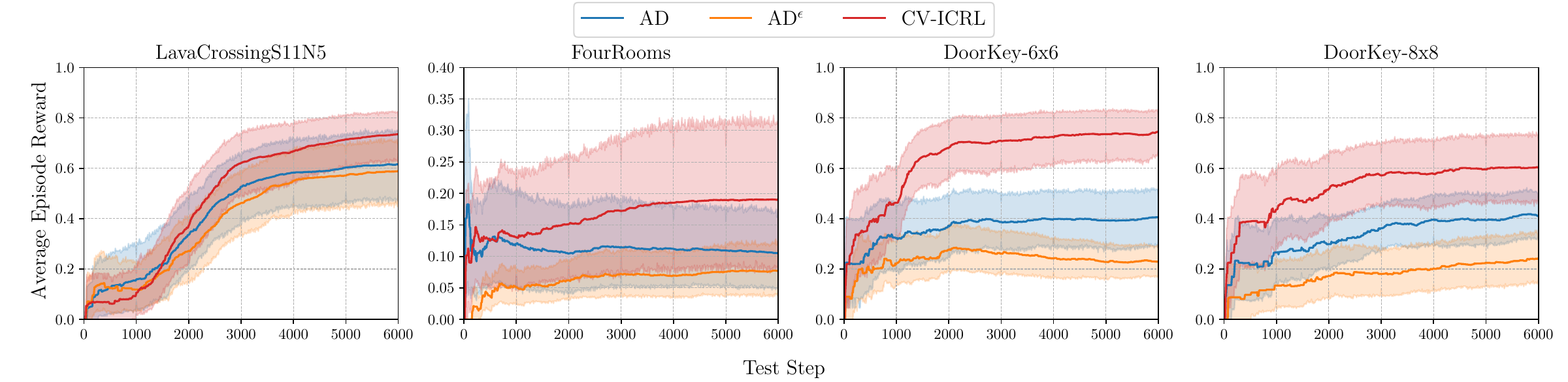}
  \vspace{-8pt}
  \caption{Cross-task-type generalization in Minigrid: models are trained on 6 task types and evaluated on 4 unseen task types.}
  \label{fig:cross task}
  \vspace{-10pt}
\end{figure*}

\begin{figure*}[t]
  \centering
  \includegraphics[width=0.85\textwidth]{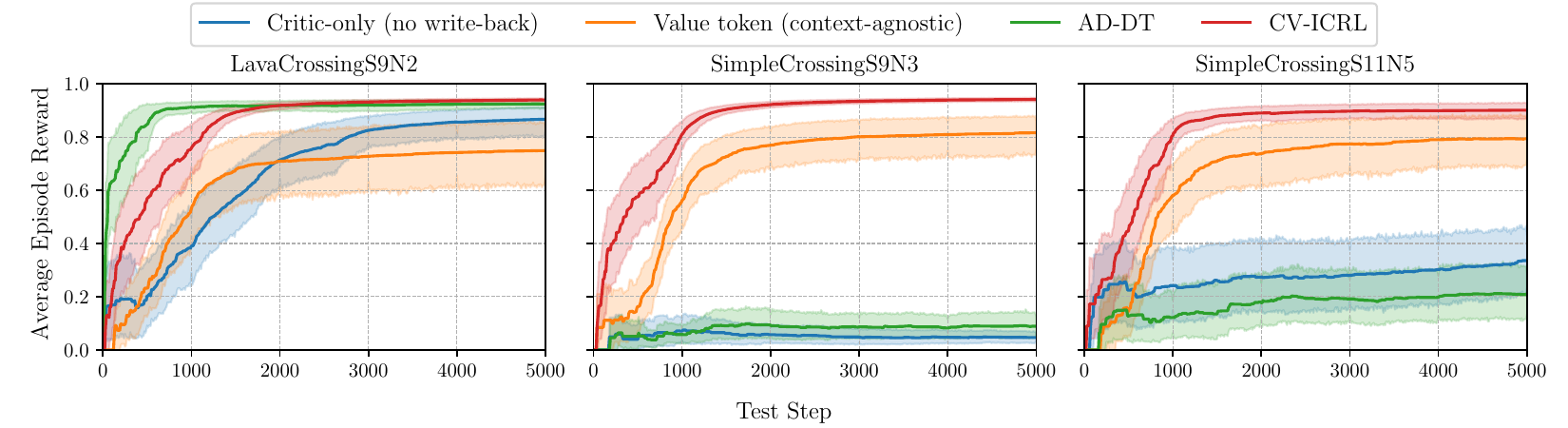}
  \vspace{-7pt}
  \caption{Ablation results on Minigrid. CV-ICRL outperforms variants without value write-back, without task-dependent value information, or with RTG conditioning, showing that gains come from feeding a meaningful context-value estimate back into the policy.}
  \label{fig:ablation}
  \vspace{-10pt}
\end{figure*}

\paragraph{Additional baseline comparisons.}
CV-ICRL achieves the highest mean AER on all four task types, while Reinformer is competitive except on LavaCrossingS9N3. Full AER, LER, and IF results are provided in Table~\ref{tab:cr_appendix_baselines_full}, with the baseline protocols in Appendix~\ref{appendix:additional_baselines}.

\paragraph{ICRL algorithm demonstrates generalization across varied tasks and scenes, and CV-ICRL enhances this ability.}

In prior works, ICRL methods have typically been validated in scenarios where the differences between tasks lie mainly in reward function or simple changes in the transition. Such simple modifications often result in tasks that are not significantly different from those seen in the training set, leading some work to suggest that these ICRL algorithms cannot address out-of-distribution generalization problems~\citep{generalization}. In this study, we aim to validate the OOD generalization capability in the Minigrid environment. We combine data from 6 types of tasks mentioned previously into a single dataset, train a general ICRL policy, and then evaluate its performance on 4 novel types of tasks (introduced in Appendix~\ref{Appendix:env}).

As shown in \Cref{fig:cross task}, the performance on these four novel task types demonstrates generalization to scenarios not encountered during pretraining. Here, ``learn-to-learn'' refers to improvement from early to later episodes through evolving test-time context while model parameters remain frozen. For example, retained context can record a failed DoorKey attempt or an unsuccessful Lava route and inform later episodes.
Moreover, CV-ICRL outperforms other AD-like methods, showing faster adaptation to novel tasks and achieving a higher average episode return.
\subsection{Ablations}
We conduct ablations to separate three possible sources of improvement: auxiliary value supervision, the presence of an additional scalar token, and DT-style scalar conditioning. 
\textbf{Critic-only (no write-back)} trains the value head with the same regression target but removes \(v_t\) from the context, so the policy cannot condition on the critic prediction.
\textbf{Value token (context-agnostic)} keeps the value-token channel but replaces the predicted context value with a fixed timestep function \(v_t \leftarrow g(t)\), independent of the task, state, and rewards.
\textbf{AD-DT (RTG conditioning)} follows the AD training protocol but replaces reward tokens with return-to-go (RTG) tokens, testing whether a DT-style scalar conditioning signal is sufficient. 

As shown in~\Cref{fig:ablation}, all three variants underperform full CV-ICRL. The critic-only variant supports a contribution beyond multitask supervision from writing value predictions back into the context. The context-agnostic token variant shows that an additional timestep channel alone does not match full CV-ICRL: it exposes interaction progress but cannot distinguish histories at the same step. AD-DT is competitive on LavaCrossingS9N2 but weaker on the two SimpleCrossing task types. With sparse terminal rewards and desired return $G^*=1$, its RTG remains nearly constant before termination, providing little graded feedback before success. Together, these task-dependent results support our central design choice of feeding a context-conditioned value estimate back into the in-context policy.

\paragraph{Deployment value-token interventions.}
Over 8,000 interaction steps, zeroing or freezing the written value token reduces AER on both SimpleCrossingS9N3 and BlockedUnlockPickup, supporting behavioral reliance on token content and continued write-back. Full results and intervention definitions are provided in Table~\ref{tab:cr_appendix_token_interventions} and Appendix~\ref{appendix:token_interventions}.

\paragraph{Alternative value supervision.}
We also train with the source PPO critic target $y_{\mathrm{PPO}}(s_t,C_t)=V^{\pi_{k(C_t)}}(s_t)$, where the context identifies the source checkpoint. It improves AER on the two SimpleCrossing task types, while checkpoint-return supervision performs better on the other four. The complete six-task comparison is provided in Table~\ref{tab:cr_appendix_ppo_state} (Appendix~\ref{appendix:alternative_supervision}).

\section{Conclusion}
In this paper, we frame ICRL as a policy-gradient-like learning process and relate its performance instability to high-variance return-based credit signals.
To mitigate this issue, we propose \textbf{CV-ICRL}, which equips the in-context policy with a learned \emph{context value} prediction and writes it back into the context, enabling TD-style, actor--critic credit assignment during inference time.
Across Dark Room, Minigrid, and Procgen, CV-ICRL yields more stable test time adaptation and improves cross-episode returns.
\textbf{Limitations.} Our policy-gradient view is supported by recent results on implicit weight adaptation in Transformers but is not yet a formal equivalence. Our cross-task-type generalization experiments are mainly within Minigrid, with small-scale MetaWorld evidence; broader environment-family shifts remain future work. More refined context-value estimation may further improve performance.

\clearpage
\section*{Impact Statement}

This paper presents work whose goal is to advance the field of in-context reinforcement learning. Our method and algorithm do not involve any adversarial attack and will not endanger human security.
All of our experiments are performed in the simulation environment, which does not involve ethical and fair issues.

\section*{Acknowledgment}
This work was supported by the Exploratory Research Grant of the Shanghai Innovation Institute.
The authors declare no competing interests.

\bibliographystyle{plainnat}
\bibliography{example_paper}

\newpage
\appendix
\onecolumn

\crefname{section}{Appendix}{Appendices}
\Crefname{section}{Appendix}{Appendices}

\section{Algorithm Pseudocode}
\label{sec:algorithm}

\Cref{alg:cv_icrl} summarizes the supervised pretraining and the test-time deployment of CV-ICRL. During pretraining, augmented contexts are constructed by teacher-forcing each transition with the source-checkpoint return $\hat J^{(k)}$ as the value token. At test time, the model produces $v_t$ from its critic head and writes it back into the context for subsequent steps; the in-context distribution thus evolves under the model's own value predictions.

\begin{algorithm}[h]
\caption{CV-ICRL: supervised pretraining and test-time deployment.}
\label{alg:cv_icrl}
\KwIn{Trajectory dataset $\mathcal{D}$ collected along online RL checkpoints $\{\pi^{(k)}\}_{k=1}^{K}$, each annotated with the checkpoint return $\hat J^{(k)}$; loss weight $\lambda$.}
\KwOut{Pretrained Transformer with parameters $\theta$, providing actor head $\pi_\theta$ and critic head $v_\theta$.}
\textbf{Stage~1: Supervised pretraining.}\\
Form augmented contexts $\bar C_t = ((s_0,a_0,r_0,v_0),\dots,(s_{t-1},a_{t-1},r_{t-1},v_{t-1}))$ with teacher-forced $v_i = \hat J^{(k_i)}$ from each sample's source checkpoint.\\
\While{not converged}{
  Sample $(\bar C_t, s_t, a_t, \hat J^{(k_t)}) \sim \mathcal{D}$.\\
  $\mathcal{L} \leftarrow -\log \pi_\theta(a_t \mid s_t, \bar C_t) + \lambda\,\bigl(v_\theta(s_t, \bar C_t) - \hat J^{(k_t)}\bigr)^{2}$.\\
  Update $\theta$ by gradient descent on $\mathcal{L}$.
}
\textbf{Stage~2: Test-time deployment on an unseen task.}\\
Reset environment to obtain $s_0$; initialize $\bar C_0 \leftarrow \emptyset$.\\
\For{$t = 0, 1, \dots, T-1$}{
  $v_t \leftarrow v_\theta(s_t, \bar C_t)$ \tcp*{predict context value}
  $a_t \sim \pi_\theta(\cdot \mid s_t, \bar C_t)$ \tcp*{sample action}
  Execute $a_t$, observe $r_t$ and $s_{t+1}$.\\
  $\bar C_{t+1} \leftarrow U\bigl(\bar C_t, (s_t, a_t, r_t, v_t)\bigr)$ \tcp*{write $v_t$ back}
  Retain only the most recent context window in $\bar C_{t+1}$.\\
  \If{the episode terminates or is truncated}{
    Reset the environment to obtain the next state $s_{t+1}$, retaining $\bar C_{t+1}$.
  }
}
\end{algorithm}

\section{Proof of Property~\ref{prop:rb_var_reduction}}
\label{app:rb_proof}

\begin{proof}[Full proof of Property~\ref{prop:rb_var_reduction}]

\noindent\textbf{Step 1: Rao--Blackwell variance reduction for $G_t$.}
Throughout this proof, $C_t$ denotes the current context; for CV-ICRL it should be read as the augmented context $\bar C_t$ that contains previously written value tokens. We omit the bar for readability.
Let $H_t=(C_t,s_t,a_t,r_t,s_{t+1},C_{t+1})$, where $C_{t+1}$ is deterministically obtained by the relevant context update, and define
$\tilde G_t \triangleq \mathbb{E}[G_t\mid H_t]$.
By the law of total variance,
\begin{equation}
\label{eq:total_var_appendix}
\mathrm{Var}(G_t)
=
\mathrm{Var}(\mathbb{E}[G_t\mid H_t])
+
\mathbb{E}[\mathrm{Var}(G_t\mid H_t)].
\end{equation}
Since $\mathbb{E}[\mathrm{Var}(G_t\mid H_t)]\ge 0$, it follows immediately that
\begin{equation}
\label{eq:rb_var_appendix}
\mathrm{Var}(G_t)\ge \mathrm{Var}(\mathbb{E}[G_t\mid H_t])=\mathrm{Var}(\tilde G_t).
\end{equation}

Under assumption~\eqref{eq:cv_td_target}, the ideal conditional expectation admits the explicit form
$\tilde G_t=r_t+\gamma V(s_{t+1},C_{t+1})$.
CV-ICRL implements this ideal target with the critic prediction $v_{t+1}\approx V(s_{t+1},C_{t+1})$.
This identification is not needed for the inequality~\eqref{eq:rb_var_appendix} itself, but is used to connect the Rao--Blackwellized target to the bootstrapped TD form.

\noindent\textbf{Step 2: Conditional Rao--Blackwellization given $(C_t,s_t,a_t)$.}
Let $Z_t \triangleq (C_t,s_t,a_t)$.
Because $\tilde G_t=\mathbb{E}[G_t\mid H_t]$ and $H_t$ contains $Z_t$ as a subvector, we also have the conditional Rao--Blackwell identity
\begin{equation}
\label{eq:cond_rb_identity}
\mathbb{E}[G_t\mid H_t] = \mathbb{E}[\,\mathbb{E}[G_t\mid H_t] \mid H_t\,]
\quad\text{and}\quad
\mathbb{E}[G_t\mid Z_t] = \mathbb{E}[\tilde G_t\mid Z_t].
\end{equation}
Moreover, applying the law of total variance conditionally on $Z_t$ yields
\begin{equation}
\label{eq:cond_total_var}
\mathrm{Var}(G_t\mid Z_t)
=
\mathrm{Var}(\mathbb{E}[G_t\mid H_t]\mid Z_t)
+
\mathbb{E}[\mathrm{Var}(G_t\mid H_t)\mid Z_t]
\ge
\mathrm{Var}(\tilde G_t\mid Z_t).
\end{equation}

\noindent\textbf{Step 3: Multiplying by the score term.}
By definition, $\psi_t=\nabla_\theta\log \pi_\theta(a_t\mid C_t,s_t)$ is a function of $(C_t,s_t,a_t)$ only, hence $\psi_t$ is $Z_t$-measurable.
Therefore, conditioned on $Z_t$, $\psi_t$ is deterministic.
For any scalar random variable $X$ (or for each coordinate of a vector), we have
\begin{equation}
\label{eq:scale_var}
\mathrm{Var}(\psi_t X \mid Z_t)
=
\|\psi_t\|^2\,\mathrm{Var}(X\mid Z_t),
\end{equation}
where $\|\psi_t\|^2$ denotes the squared norm (and the identity holds componentwise in the vector case).
Applying~\eqref{eq:scale_var} with $X=G_t$ and $X=\tilde G_t$, and using~\eqref{eq:cond_total_var}, gives
\begin{equation}
\label{eq:rb_score_appendix}
\mathrm{Var}(\psi_t G_t \mid Z_t)
=
\|\psi_t\|^2\,\mathrm{Var}(G_t\mid Z_t)
\ge
\|\psi_t\|^2\,\mathrm{Var}(\tilde G_t\mid Z_t)
=
\mathrm{Var}(\psi_t \tilde G_t \mid Z_t),
\end{equation}
which is exactly~\eqref{eq:rb_var_score_prop_return} (recalling $Z_t=(C_t,s_t,a_t)$).

This completes the proof.
\end{proof}

\paragraph{Approximate value predictions.}
Let $v_{t+1}=V_{t+1}+e_{t+1}$ and $\tilde G_t=r_t+\gamma V_{t+1}$. The practical target is
\[
\widehat G_t^{\mathrm{TD}}=\tilde G_t+\gamma e_{t+1}.
\]
Under the same distribution as the ideal target, its mean shift is $\gamma\mathbb E[e_{t+1}]$, and
\[
\operatorname{Var}(\widehat G_t^{\mathrm{TD}})
=
\operatorname{Var}(\tilde G_t)
+\gamma^2\operatorname{Var}(e_{t+1})
+2\gamma\operatorname{Cov}(\tilde G_t,e_{t+1}).
\]
Thus the practical target retains a variance advantage over $G_t$ whenever
\[
\gamma^2\operatorname{Var}(e_{t+1})
+2\gamma\operatorname{Cov}(\tilde G_t,e_{t+1})
<
\operatorname{Var}(G_t)-\operatorname{Var}(\tilde G_t).
\]
Controlled approximation error can therefore preserve the ideal Rao--Blackwell variance gain; a nonzero mean error introduces a mean shift.

\section{Improved Performance Bound Conditioned on Context Value}
\label{Appendix:improved_performance_bound}

\paragraph{Notation in this appendix.}
We carefully distinguish the \emph{ideal} (true) context value from the \emph{estimated} value token used by CV-ICRL.
For a policy $\pi$, the context value keeps the definition from~\Cref{sec:context_value}: $V^\pi(s;C)$ is the expected discounted return conditioned on the current state and context.
In the distillation analysis below, however, the scalar value token is supervised by a checkpoint-level return estimate rather than by exact per-state values.
We denote this proxy target by $y^*(C)$.
For the reference source policy $\pi^*\in\Pi$ associated with the current training context,
\[
y^*(C) \triangleq J(\pi^*),
\]
a deterministic scalar functional of $\pi^*$ alone, where $J(\pi^*)$ is the Monte-Carlo evaluated expected return of that checkpoint policy.
Thus $y^*(C)$ should be understood as the ideal value-token target used in this analysis, not as a separate definition of the state--context value $V^{\pi^*}(s;C)$.
Lowercase $v(C)$ denotes the \emph{estimated} context value produced by our value head and written into the context; its estimation error is
\[
d_v \triangleq |\,v(C) - y^*(C)\,|.
\]
The ideal regime $v=y^*$ (i.e., $d_v=0$) is treated first in~\Cref{lemma:2}; the realistic regime $d_v>0$ is handled in~\Cref{cor:value_error}.

\begin{property}[Improved Performance Bound]
\label{prop:tighter_bound_value_short}
Let $\pi(\cdot\mid s,C)$ denote the base ICRL policy and let $\pi_v(\cdot\mid s,C,v)$ be its value-conditioned variant (CV-ICRL), with $v(C)$ the predicted context-value token. Let $\pi^*(\cdot\mid s,C)$ be the reference (source) policy. Define the worst-case TV deviations
\[
\epsilon_{\mathrm{base}} \triangleq \sup_{C} D_{TV}\!\big(\pi^*(\cdot\mid s,C)\,\Vert\,\pi(\cdot\mid s,C)\big),
\qquad
\epsilon_v \triangleq \sup_{C} D_{TV}\!\big(\pi^*(\cdot\mid s,C)\,\Vert\,\pi_v(\cdot\mid s,C,v)\big).
\]
Then for any state $s$,
\begin{align}
\sup_C \big|V^{\pi^*}(s;C)-V^{\pi}(s;C)\big|
&\le
\frac{2r_{\max}}{(1-\gamma)^2}\,\epsilon_{\mathrm{base}}, \\
\sup_C \big|V^{\pi^*}(s;C)-V^{\pi_v}(s;C,v)\big|
&\le
\frac{2r_{\max}}{(1-\gamma)^2}\,\epsilon_v.
\end{align}
Furthermore, under the posterior model of~\Cref{lemma:2}, when the estimation error of $v$ satisfies $d_J - 2d_v > 0$ (with $d_J$ defined in~\eqref{eq:dJ_def};~\Cref{cor:value_error}), conditioning on $v$ strictly increases the posterior mass on $\pi^*$, yielding $\epsilon_v<\epsilon_{\mathrm{base}}$ and a strictly tighter performance bound for $\pi_v$.
The proof combines~\Cref{theorem:2,lemma:1,lemma:2} and is given at the end of this appendix.
\end{property}

To formally quantify the advantage of incorporating value information, we must first establish the baseline performance. We begin by presenting the theoretical guarantee on the performance error for a standard ICRL policy, which is learned without value side-information. This result, adapted from imitation learning works ~\citep{hebridging}, serves as the foundation for our subsequent comparison.

Here, for different pretrained data paris $(C_t,s_t, a_t)$ (with $v_t$ for our methods), their source policies $\pi^{(k)}_t$ are different. Here, we formulate the learning target policy condition on $C_t$ to be $\pi^*(a|s,C_t)$, where $\pi^*(a|s,C_t) = \pi^{(k)}_t(a|s)$. 

\begin{lemma}[Performance Bound of ICRL Policy]
\label{theorem:2}
The suboptimality gap for the ICRL policy, defined as the worst-case difference in policy value between the learned ICRL policy $\pi(\cdot|s, C)$ and the empirical optimal policy (target policy) $\pi^*(\cdot|s, C)$, is bounded as follows:
    \begin{equation}
        \sup_C |V^{\pi^*}(s; C) - V^\pi(s;C)| \le \frac{2r_{\text{max}}}{(1-\gamma)^2}\epsilon
    \end{equation}
    where $\epsilon$ represents the worst-case statistical divergence between the two policies:

\begin{equation}
\label{eq:eps}
\epsilon = \sup_{C}D_{TV}(\pi^*(\cdot|s,C) \Vert \pi(\cdot|s,C))
\end{equation}
\end{lemma}
For notational simplicity, we use $V(s;C)$ to denote $V^{\pi}(s;C)$, and use $V^*(s,C)$ to denote $V^{\pi^*}(s;C)$ in the remainder of this proof.
\begin{proof}

Expanding $V(s;C)$ for one step:
\begin{equation}
    V(s;C) = \sum_{a,s'} T(s'|s, a)\pi(a|s)[R(s, a) + \gamma V(s';C'))]
\end{equation}
Here, $C' = C \cup \{(a,r,s')\}$.

Thus,
\begin{align}
    & V^*(s;C) - V(s;C) \nonumber \\
    =& \sum_{a,s'} T(s'|s, a)[\pi^*(a|s)[R(s, a) + \gamma V^*(s';C'))] - \pi(a|s)[R(s, a) + \gamma V(s';C'))]] \nonumber \\
    =& \sum_{a,s'} T(s'|s, a)R(s, a)[\pi^*(a|s)- \pi(a|s)] \nonumber \\ &+\gamma\sum_{a,s'} T(s'|s, a)[\pi^*(a|s)V^*(s';C') - \pi(a|s) V(s';C')] \nonumber \\
    =& \sum_{a,s'} T(s'|s, a)R(s, a)[\pi^*(a|s)- \pi(a|s)] \nonumber \\ 
    &+\gamma\sum_{a,s'} T(s'|s, a)\pi^*(a|s)[V^*(s';C') -  V(s';C')] + V(s';C')[\pi^*(a|s) - \pi(a|s)] \nonumber \\
    =& \sum_{a,s'} T(s'|s, a)[R(s, a) + \gamma V(s';C')][\pi^*(a|s)- \pi(a|s)] \nonumber \\&+\gamma\sum_{a,s'} T(s'|s, a)\pi^*(a|s)[V^*(s';C') -  V(s';C')] \nonumber \\
    \le &\ \Vert T(R+\gamma V)\Vert_\infty\Vert\pi^* - \pi\Vert_\infty + \gamma\Vert V^* - V \Vert_\infty \quad\quad\quad  \triangleright \text{ H\"{o}lder's inequality}\nonumber \\
    \le &\ 2\Vert (R+\gamma V)\Vert_\infty\epsilon + \gamma\Vert V^* - V \Vert_\infty \quad\quad\quad\quad\quad  \triangleright \text{\Cref{eq:eps}}\nonumber \\
    \le &\ \frac{2r_{\text{max}}}{1-\gamma}\epsilon + \gamma\Vert V^* - V \Vert_\infty \quad\quad\quad\quad\quad\quad\quad\quad\quad  \triangleright V(s)\le \frac{r_{\text{max}}}{1-\gamma}\nonumber
\end{align}

Then,
\begin{align*}
    \Vert V^* - V \Vert_\infty \le &\ \frac{2r_{\text{max}}}{1-\gamma}\epsilon + \gamma\Vert V^* - V \Vert_\infty \\
    (1-\gamma)\Vert V^* - V \Vert_\infty \le &\ \frac{2r_{\text{max}}}{1-\gamma}\epsilon
\end{align*}

Then we have
\begin{equation}
    \sup_C |V^*(s;C) - V(s;C)| = \Vert V^* - V \Vert_\infty \le \frac{2r_{\text{max}}}{(1-\gamma)^2}\epsilon
\end{equation}

\end{proof}

We treat the inference process of an ICRL model as a two-stage sampling procedure: first, sampling a candidate source policy $\pi\in\Pi$ from the posterior $P(\pi\mid C)$ induced by the context, and second, sampling an action $a$ from the chosen policy. The resulting in-context distribution is the posterior mixture
\begin{equation}
\label{eq:posterior_mixture}
    \pi(\cdot|s, C) = \sum_{\pi\in\Pi}P(\pi|C)\,\pi(\cdot|s).
\end{equation}
Analogously, conditioning on the value token $v$ yields $\pi(\cdot|s,C,v) = \sum_{\pi\in\Pi}P(\pi\mid C,v)\,\pi(\cdot|s)$.

\begin{lemma}[TV bound via posterior mass]
\label{lemma:1}
Assume the in-context distribution decomposes via the posterior mixture~\eqref{eq:posterior_mixture}, with $\pi^*\in\Pi$. Then for any value token $v(C)$ written into the context,
\begin{equation}
\label{eq:tv_bound_withV}
    \sup_{C}D_{TV}\!\big(\pi^*(\cdot|s, C)\,\Vert\,\pi(\cdot|s, C, v)\big) \le 1-P(\pi^*\mid C, v),
\end{equation}
\begin{equation}
\label{eq:tv_bound_noV}
    \sup_{C}D_{TV}\!\big(\pi^*(\cdot|s, C)\,\Vert\,\pi(\cdot|s, C)\big) \le 1-P(\pi^*\mid C).
\end{equation}
Eq.~\eqref{eq:tv_bound_withV} holds for any value token $v$; the ideal regime $v=y^*$ feeds into~\Cref{lemma:2}, while the realistic case $d_v>0$ is treated in~\Cref{cor:value_error}.
\end{lemma}

\begin{proof}
We prove~\eqref{eq:tv_bound_withV}; the proof of~\eqref{eq:tv_bound_noV} is identical.
By definition of total variation distance,
\begin{align*}
&D_{TV}\!\left(\pi^*(\cdot\mid s,C)\,\Vert\,\pi(\cdot\mid s,C,v)\right)\\
=&
\frac{1}{2}\sum_a
\left|
\pi^*(a\mid s,C) - \sum_{\pi\in\Pi} P(\pi\mid C,v)\,\pi(a\mid s)
\right|\\
=&
\frac{1}{2}\sum_a
\left|
\pi^*(a\mid s,C) - P(\pi^*\mid C,v)\pi^*(a\mid s,C)
- \sum_{\pi\in\Pi\setminus\{\pi^*\}} P(\pi\mid C,v)\,\pi(a\mid s)
\right|\\
=&
\frac{1}{2}\sum_a
\left|
(1-P(\pi^*\mid C,v))\pi^*(a\mid s,C)
- \sum_{\pi\in\Pi\setminus\{\pi^*\}} P(\pi\mid C,v)\,\pi(a\mid s)
\right|\\
\le&
\frac{1}{2}\sum_a
\Big(
(1-P(\pi^*\mid C,v))\pi^*(a\mid s,C)
+ \sum_{\pi\in\Pi\setminus\{\pi^*\}} P(\pi\mid C,v)\,\pi(a\mid s)
\Big)\\
=&
\frac{1}{2}\Big(
(1-P(\pi^*\mid C,v))\sum_a \pi^*(a\mid s,C)
+ \sum_{\pi\in\Pi\setminus\{\pi^*\}} P(\pi\mid C,v)\sum_a \pi(a\mid s)
\Big)\\
=&
\frac{1}{2}\Big(
(1-P(\pi^*\mid C,v)) + \sum_{\pi\in\Pi\setminus\{\pi^*\}} P(\pi\mid C,v)
\Big)\\
=&
\frac{1}{2}\Big(
(1-P(\pi^*\mid C,v)) + (1-P(\pi^*\mid C,v))
\Big)\\
=&
1-P(\pi^*\mid C,v),
\end{align*}
where we used $\sum_a \pi^*(a\mid s,C)=1$, $\sum_a \pi(a\mid s)=1$, and
$P(\pi^*\mid C,v)+\sum_{\pi\ne \pi^*}P(\pi\mid C,v)=1$.

Similarly, for the naive ICRL policy $\pi(\cdot|C)$ we can also have
\begin{equation}
    \sup_{C}D_{TV}(\pi^*(\cdot|s, C) \Vert \pi(\cdot|s, C)) \le 1-P(\pi^*|C)
\end{equation}

\end{proof}

Now we consider the posterior term $P(\pi|C)$. We will prove that $P(\pi^*|C, y^*)$ is larger than $P(\pi^*|C)$, so the upper bound of $D_{TV}(\pi^*(\cdot|C) \Vert \pi(\cdot|C, y^*))$ is tighter than the corresponding bound without the value token.

\begin{lemma}[Comparison of two upper bounds]
\label{lemma:2}
Let $\Pi=\{\pi_i\}_{i=1}^N$ be the finite set of source-policy checkpoints used during distillation, with a uniform prior $P(\pi_i)=1/N$ (each checkpoint is sampled with equal frequency in the ICRL training pipeline).
Assume the joint likelihood factorizes as
\begin{equation}
\label{eq:factorization}
P(C,y^*\mid \pi)=P(C\mid \pi)\,P(y^*\mid \pi),
\end{equation}
which is justified by noting that $y^*(C)=J(\pi^*)$ is a deterministic functional of $\pi^*$ alone (it does not depend on the realized context $C$ given $\pi$); hence, conditional on a candidate $\pi$, the data $C$ and the value-token observation $y^*$ act as separate likelihood channels.
We model the value likelihood by
\begin{equation}
\label{eq:value_likelihood}
P(y^*\mid \pi)\ \propto\ \exp\!\big(-\beta\,|y^*(C) - J(\pi)|\big),
\qquad \beta>0,
\end{equation}
where $J(\pi)$ is the (true) expected return of $\pi$ in the environment.

Fix a \emph{source policy} $\pi^*\in\Pi$ and take the observed value target to be
\begin{equation}
\label{eq:vc_def}
y^*(C) = J(\pi^*).
\end{equation}
Define the return separation
\begin{equation}
\label{eq:dJ_def}
d_J \triangleq \min_{i\neq *} |J(\pi^*)-J(\pi_i)|.
\end{equation}
Then
\begin{equation}
\label{eq:posterior_ratio_bound}
\frac{P(\pi^*\mid C,y^*)}{P(\pi^*\mid C)}
\ \ge\
1+\delta_{\mathrm{rel}},
\end{equation}
where
\begin{equation}
\label{eq:delta_rel_def}
\delta_{\mathrm{rel}}
=
\frac{\left(\sum_{i\neq *}P(C\mid \pi_i)\right)\left(e^{\beta d_J}-1\right)}
{P(C\mid \pi^*)e^{\beta d_J}+\sum_{i\neq *}P(C\mid \pi_i)} \ >\ 0
\quad\text{whenever } d_J>0.
\end{equation}
\end{lemma}

\begin{proof}
By Bayes' rule and the uniform prior,
\begin{equation}
\label{eq:bayes_C}
P(\pi\mid C)=\frac{P(C\mid \pi)}{\sum_j P(C\mid \pi_j)}.
\end{equation}
Using the factorization assumption~\eqref{eq:factorization},
\begin{equation}
\label{eq:bayes_CV}
P(\pi\mid C,y^*)
=\frac{P(C\mid \pi)P(y^*\mid \pi)}{\sum_j P(C\mid \pi_j)P(y^*\mid \pi_j)}.
\end{equation}
Therefore,
\begin{equation}
\label{eq:ratio_start}
\frac{P(\pi^*\mid C,y^*)}{P(\pi^*\mid C)}
=
\frac{P(y^*\mid \pi^*)\sum_j P(C\mid \pi_j)}
{\sum_i P(C\mid \pi_i)P(y^*\mid \pi_i)}.
\end{equation}

Now we lower bound the denominator using the value-likelihood model.
Under~\eqref{eq:value_likelihood} and $y^*(C)=J(\pi^*)$, we have
\begin{equation}
\label{eq:likelihood_ratio}
\frac{P(y^*\mid \pi^*)}{P(y^*\mid \pi_i)}
=
\exp\!\big(\beta\,|J(\pi^*)-J(\pi_i)|\big)
\ \ge\
\exp(\beta d_J)
\ \triangleq\
\Gamma_{\min},
\qquad \forall i\neq *,
\end{equation}
where $d_J$ is defined in~\eqref{eq:dJ_def}.
Equivalently, for all $i\neq *$,
\begin{equation}
\label{eq:upper_other_likelihood}
P(y^*\mid \pi_i)\ \le\ \frac{P(y^*\mid \pi^*)}{\Gamma_{\min}}.
\end{equation}

Plugging~\eqref{eq:upper_other_likelihood} into the denominator of~\eqref{eq:ratio_start} yields
\begin{align}
\frac{P(\pi^*\mid C,y^*)}{P(\pi^*\mid C)}
&=
\frac{P(y^*\mid \pi^*)\sum_j P(C\mid \pi_j)}
{P(C\mid \pi^*)P(y^*\mid \pi^*)+\sum_{i\neq *}P(C\mid \pi_i)P(y^*\mid \pi_i)} \nonumber\\
&\ge
\frac{P(y^*\mid \pi^*)\sum_j P(C\mid \pi_j)}
{P(C\mid \pi^*)P(y^*\mid \pi^*)+\sum_{i\neq *}P(C\mid \pi_i)\frac{P(y^*\mid \pi^*)}{\Gamma_{\min}}}\\
&=
\frac{\sum_j P(C\mid \pi_j)}
{P(C\mid \pi^*)+\sum_{i\neq *}P(C\mid \pi_i)/\Gamma_{\min}}.
\label{eq:ratio_bound_mid}
\end{align}

Let $S_* \triangleq P(C\mid \pi^*)$ and $S_{\mathrm{other}}\triangleq \sum_{i\neq *}P(C\mid \pi_i)$.
Then~\eqref{eq:ratio_bound_mid} becomes
\begin{equation}
\label{eq:ratio_S}
\frac{P(\pi^*\mid C,y^*)}{P(\pi^*\mid C)}
\ge
\frac{S_*+S_{\mathrm{other}}}{S_*+S_{\mathrm{other}}/\Gamma_{\min}}
=
1+\frac{S_{\mathrm{other}}(\Gamma_{\min}-1)}{S_*\Gamma_{\min}+S_{\mathrm{other}}}.
\end{equation}
Substituting $\Gamma_{\min}=e^{\beta d_J}$ from~\eqref{eq:likelihood_ratio} gives~\eqref{eq:delta_rel_def}, and $\delta_{\mathrm{rel}}>0$ whenever $d_J>0$.
\end{proof}

In practice the value head is not exact: the value token $v(C)$ deviates from the ideal value-token target $y^*(C)=J(\pi^*)$ by $d_v$, defined in the appendix notation. The next corollary shows that the posterior-ratio improvement of~\Cref{lemma:2} survives this error as long as $d_v$ is not too large.

\begin{corollary}[Robustness of \Cref{lemma:2} to value-estimation error]
\label{cor:value_error}
Suppose that, in the setting of~\Cref{lemma:2}, the value token observed by the model is the estimate $v(C)$ in place of the ideal target $y^*(C)$, with error $d_v\triangleq|v(C)-y^*(C)|$, and adopt the value-likelihood model
\begin{equation}
\label{eq:value_like_hat}
P(v\mid \pi)\ \propto\ \exp\!\big(-\beta\,|v(C) - J(\pi)|\big),
\qquad \beta>0.
\end{equation}
Then the posterior-ratio inequality of~\Cref{lemma:2} continues to hold with the modified constant
\begin{equation}
\label{eq:Gamma_min_hat}
\Gamma_{\min} \;\triangleq\; \exp\!\big(\beta(d_J-2d_v)\big),
\end{equation}
and in particular,
\begin{equation}
\label{eq:cond_dv}
\frac{P(\pi^*\mid C, v)}{P(\pi^*\mid C)} \;>\; 1
\qquad\text{whenever}\qquad
d_J - 2d_v > 0.
\end{equation}
\end{corollary}

\begin{proof}
For any $\pi_i\neq \pi^*$, the triangle inequality gives
\begin{align}
\label{eq:lower_bound_hat_gap}
|v(C) - J(\pi_i)|
&=
|(v(C) - J(\pi^*)) + (J(\pi^*) - J(\pi_i))|
\nonumber\\
&\ge
|J(\pi^*) - J(\pi_i)| - |v(C) - J(\pi^*)|
=
|J(\pi^*) - J(\pi_i)| - d_v.
\end{align}
Plugging into~\eqref{eq:value_like_hat}, the likelihood ratio satisfies
\begin{align}
\Gamma_i
\;\triangleq\;
\frac{P(v\mid \pi^*)}{P(v\mid \pi_i)}
&=
\exp\!\Big(\beta\big(|v(C) - J(\pi_i)| - |v(C) - J(\pi^*)|\big)\Big)
\nonumber\\
&\ge
\exp\!\Big(\beta\big(|J(\pi^*) - J(\pi_i)| - 2d_v\big)\Big).
\label{eq:Gamma_lower_hat}
\end{align}
Taking the minimum over $i\neq *$ yields the uniform lower bound $\Gamma_i\ge \exp(\beta(d_J-2d_v))=\Gamma_{\min}$. Substituting this $\Gamma_{\min}$ into the proof of~\Cref{lemma:2} preserves the chain leading to~\eqref{eq:posterior_ratio_bound}, and the resulting $\delta_{\mathrm{rel}}$ is strictly positive whenever $\Gamma_{\min}>1$, i.e., $d_J-2d_v>0$.
\end{proof}

\begin{proof}[Proof of \Cref{prop:tighter_bound_value_short}]
The two suboptimality-gap inequalities follow directly from~\Cref{theorem:2} applied to $\pi$ and to $\pi_v$ respectively, with the worst-case TV deviations $\epsilon_{\mathrm{base}}$ and $\epsilon_v$ as defined in the statement of~\Cref{prop:tighter_bound_value_short}.

For the strictness claim, \Cref{lemma:1} gives
\[
\epsilon_{\mathrm{base}} \;\le\; 1-P(\pi^*\mid C),
\qquad
\epsilon_v \;\le\; 1-P(\pi^*\mid C, v).
\]
Under the posterior model of~\Cref{lemma:2} with the ideal value token ($v=y^*$, $d_v=0$), we have $P(\pi^*\mid C,v)\ge (1+\delta_{\mathrm{rel}})P(\pi^*\mid C) > P(\pi^*\mid C)$ whenever $d_J>0$. \Cref{cor:value_error} extends this to the realistic case in which the value head incurs error $d_v$, with the strict inequality preserved as long as $d_J - 2d_v > 0$.
In either case,
\[
P(\pi^*\mid C, v) \;>\; P(\pi^*\mid C)
\qquad\Longrightarrow\qquad
\epsilon_v \;<\; \epsilon_{\mathrm{base}},
\]
and multiplying through by $2r_{\max}/(1-\gamma)^2$ yields a strictly tighter performance bound for the value-conditioned policy $\pi_v$.
\end{proof}

\section{Details of Experiments}\ \\
\subsection{Environments}
\label{Appendix:env}
\subsubsection{Dark Room}
The Dark Room environment is a challenging testbed for an agent's ability to perform efficient exploration under conditions of extreme reward sparsity and partial observability. The task places an agent in a large gridworld with a single goal state, but the agent's perception is limited to its local vicinity. A positive reward is only granted upon reaching the goal, meaning the agent must conduct a systematic, memory-based search to explore the space without any guiding signals. It is therefore highly effective at evaluating an agent's capacity to use memory for long-term navigation and exploration.

\subsubsection{Minigrid}

The environments commonly employed in current methods, such as Dark Room and MiniWorld, are relatively simple. For instance, the state space in Dark Room is merely two-dimensional. More importantly, the task variations in these settings are largely confined to changes in the reward function or minor shifts in dynamics. This is insufficient to demonstrate the generalization capability of an ICRL algorithm to unseen tasks. To address this, we use Minigrid, a more complex and diverse environment, to demonstrate that our ICRL algorithm can indeed generalize to completely new types of tasks.

Minigrid environment features a collection of gridworld scenarios where an agent must infer and accomplish a goal through exploration. Tasks range from simple navigation, such as bypassing obstacles to reach a goal, to complex sequential decision-making, like opening a series of doors. Critically, Minigrid provides a fixed-size (7x7), uniform symbolic observation across all its diverse tasks. This standardization removes the challenge of unifying observation representations, making it an ideal platform to directly test how effectively an In-Context Reinforcement Learning (ICRL) policy can adapt to new tasks.

Below are the introductions of 6 task types used for both training and testing.

\paragraph{LavaCrossingS9N3.} In this task type, the agent must navigate through a room with deadly lava streams running horizontally or vertically. The agent must reach the green goal square while avoiding lava, with a single safe crossing point for each stream. A path to the goal is guaranteed.

\paragraph{LavaCrossingS9N2.} Similar to LavaCrossingS9N3, but the challenge is slightly less difficult compared to LavaCrossingS9N3.

\paragraph{SimpleCrossingS9N3.} In this task type, the agent must navigate through a room with walls instead of lava. The objective is to reach the green goal square while avoiding walls. This task is easier compared to the LavaCrossing tasks and is useful for quickly testing algorithms.

\paragraph{SimpleCrossingS11N5.} Similar to SimpleCrossingS9N3, the agent must reach the green goal square while avoiding walls. The task involves a larger environment with more walls, making it moderately more challenging than SimpleCrossingS9N3 but still easier than the LavaCrossing tasks.

\paragraph{BlockedUnlockPickup.} In this task type, the agent must pick up an object placed in another room, behind a locked door. The door is blocked by a ball that the agent must first move to unlock the door. The agent must learn to move the ball, pick up the key, open the door, and then pick up the object in the other room.

\paragraph{Unlock.} This task type is a simplified version of BlockedUnlockPickup. The agent only needs to pick up the key and then unlock the door.

\begin{figure}[h]
    \centering
    \includegraphics[width=0.8\linewidth]{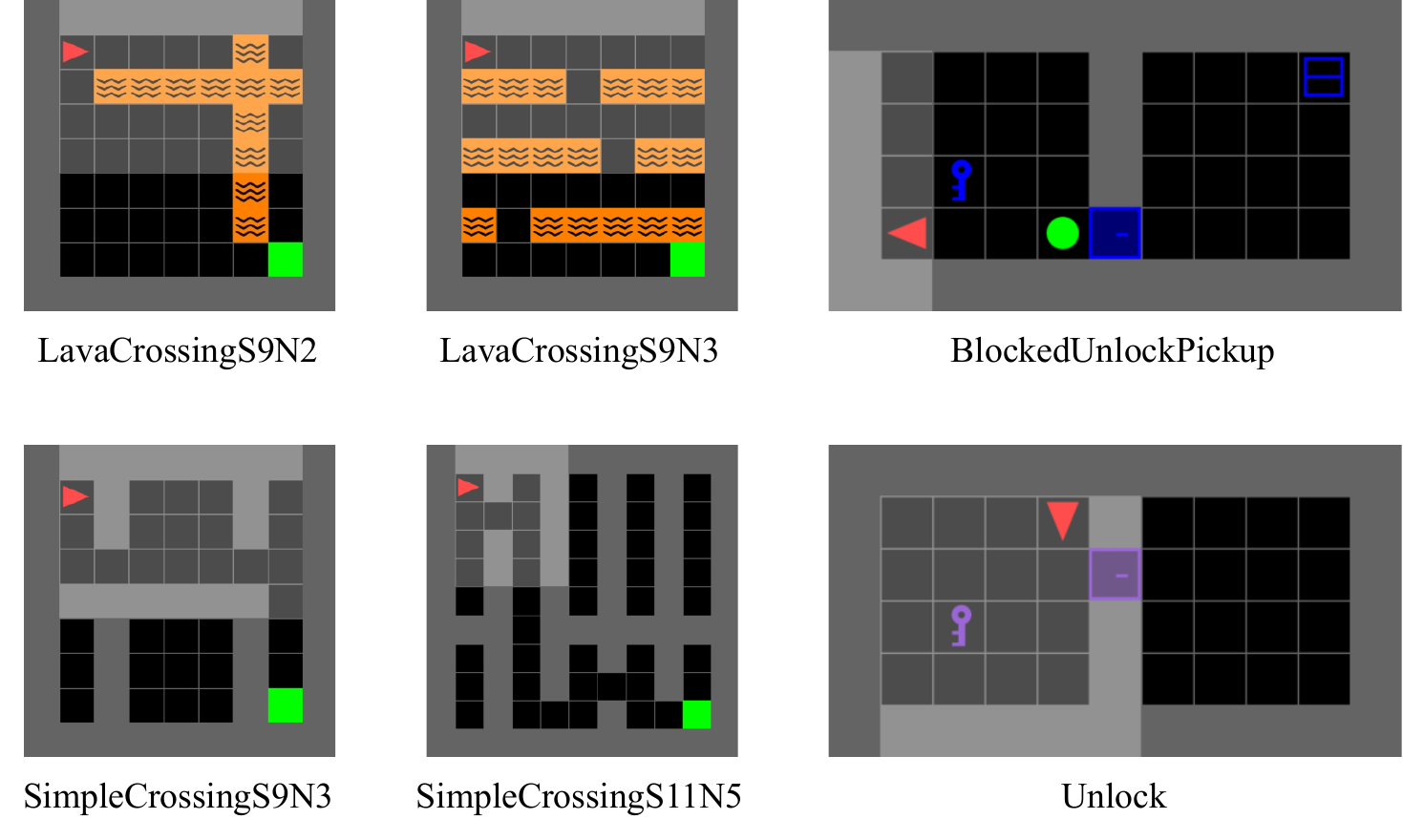}
    \caption{The 6 tasks for the main experiments, with different seeds corresponding to different layouts.}
    \label{fig:minigrid_main_envs}
\end{figure}

Below are the introductions of 4 task types used for the cross-task-type generalization experiments, and only for test.

\paragraph{LavaCrossingS11N5.} Similar to LavaCrossingS9N3, LavaCrossingS11N5 introduces greater difficulty by increasing the number of lava streams and the complexity of the room layout. This makes it more challenging for the agent to find a safe path to the goal.

\paragraph{FourRooms.} Agent must navigate a maze composed of four rooms, interconnected by gaps in the walls. The agent's goal is to reach the green goal square to receive a reward. Both the agent and the goal square are randomly placed in any of the four rooms.

\paragraph{DoorKey-6x6.} In this task type, the agent must pick up a key to unlock a door and then reach the green goal square. 

\paragraph{DoorKey-8x8.} Similar to DoorKey-6x6, but is more complex due to the larger environment size.

\begin{figure}[h]
    \centering
    \includegraphics[width=0.4\linewidth]{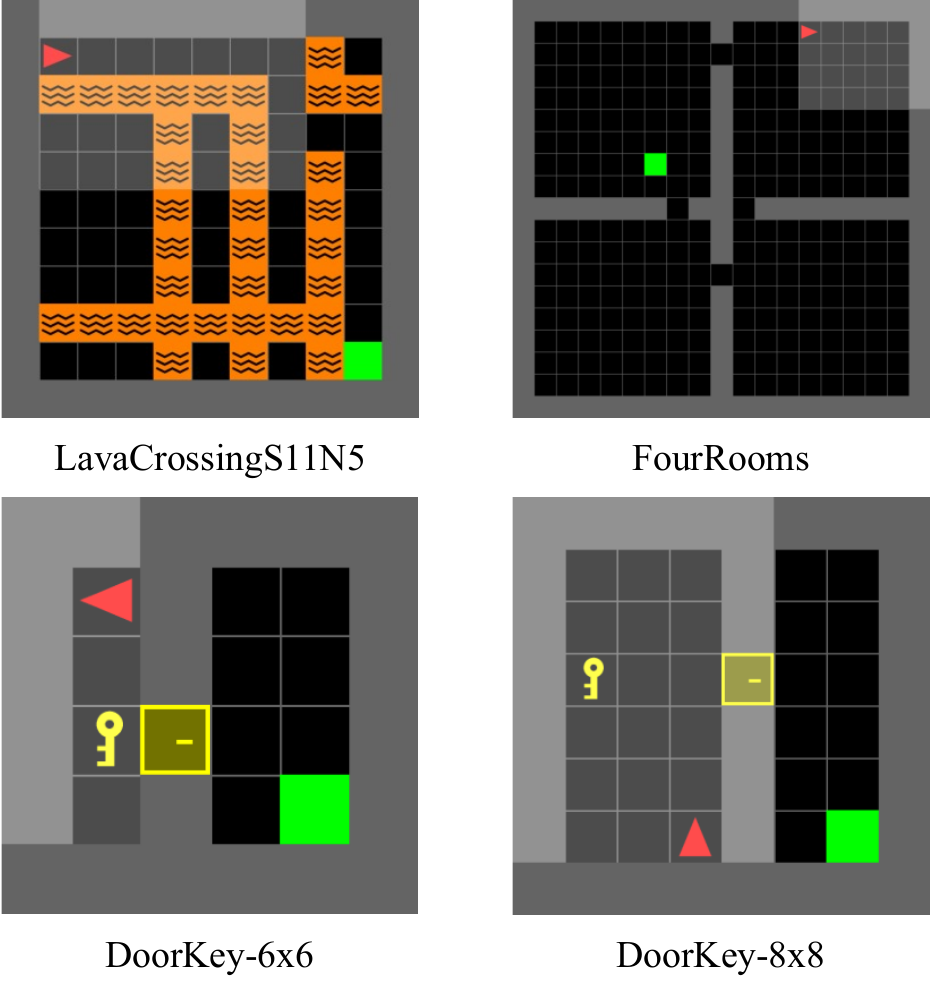}
    \caption{The 4 tasks for the cross-task generalization experiments, with different seeds corresponding to different layouts.}
    \label{fig:minigrid_cross_task_envs}
\end{figure}

\subsubsection{Procgen}

Procgen~\citep{cobbe2019quantifying} is a suite of procedurally generated arcade-style environments designed to evaluate generalization in deep RL. Each task type uses 64$\times$64 RGB observations, a 15-dimensional discrete action space, and procedurally generated levels that vary in layout, dynamics, and visual appearance. Compared with gridworld testbeds (Dark Room, Minigrid), Procgen poses substantially greater perceptual and dynamical challenges. We evaluate on three Procgen task types described below.

\paragraph{Bigfish.} The agent controls a small fish that must consume smaller fish to grow while avoiding larger fish that would consume it; the goal is to maximize cumulative growth.

\paragraph{Starpilot.} A side-scrolling shoot-em-up where the agent dodges enemies and projectiles while shooting incoming targets to maximize the score.

\paragraph{Miner.} The agent navigates a tile-based mining environment, collects diamonds, and exits without being killed by falling boulders.

\begin{figure}[h]
    \centering
    \includegraphics[width=0.9\linewidth]{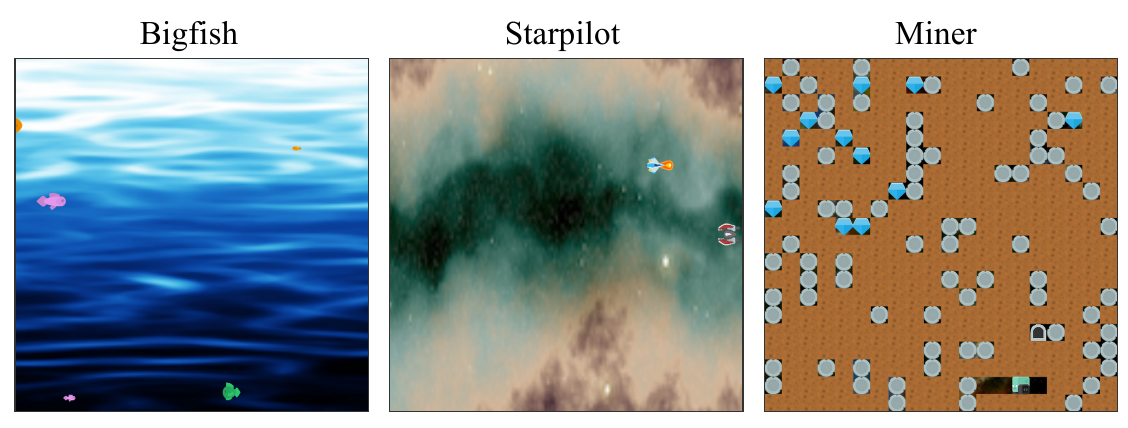}
    \caption{The three Procgen task types used in our experiments: Bigfish, Starpilot, and Miner.}
    \label{fig:procgen_env}
\end{figure}

\subsubsection{MetaWorld}

MetaWorld is a benchmark of robotic manipulation tasks designed for continuous-control reinforcement learning. Unlike Dark Room, Minigrid, and Procgen, where the policy selects from a discrete action set, MetaWorld requires the agent to output continuous motor commands for controlling a simulated robotic arm. This makes the environment qualitatively different from our main testbeds: adaptation must account for smooth control, contact dynamics, and precise object manipulation rather than only selecting among symbolic or discrete actions.

We use MetaWorld as a small-scale stress test for whether CV-ICRL can transfer beyond discrete-control domains. Specifically, we evaluate on two ML1 tasks, \texttt{ML1-reach-v3} and \texttt{ML1-push-v3}. In \texttt{ML1-reach-v3}, the agent must move the robot end-effector to a target position; in \texttt{ML1-push-v3}, it must push an object toward a goal. These tasks provide a compact continuous-control setting for probing whether context-value conditioning remains useful when the action space and dynamics differ substantially from the environments used in the main experiments.

\subsection{Preparations of Training Datasets}
\label{Appendix:pretrain}
For the Dark Room environment, we train PPO in the same way as Algorithm Distillation. We choose 20 different environments, each with different goal positions (x, y), and use these to train the PPO agent.

For Minigrid, the environment is more challenging than Dark Room. The PPO policy faces instability when trained from scratch for each layout. To address this, we adopt a pretrain-finetune strategy: we first pretrain a PPO model on a variety of layouts and then finetune it on specific layouts as needed.

For Procgen, we train a separate PPO model per task type (Bigfish, Starpilot, Miner) on procedurally generated levels, using a CNN-based policy network suited to the 64$\times$64 RGB observations.

The PPO algorithm we use is based on the implementation provided by Stable-Baselines3, a popular library for reinforcement learning algorithms. This implementation ensures stability and efficiency during training. Below is a table summarizing the hyperparameters used for training our PPO model. The values not mentioned are set to the default values from Stable-Baselines3.

\begin{table}[h!]
\centering
\caption{Hyperparameters of PPO training for Dark Room}
\label{tab:ppo_darkroom_hyperparameters}
\begin{tabular}{c|c}
\toprule
\textbf{Hyperparameter} & \textbf{Value} \\
\midrule
Learning Rate & 0.0001 \\
Policy Network Architecture & MlpPolicy \\
Training Steps & 1000000 \\
Save Frequency & 50000\\
N Steps & 500\\
\bottomrule
\end{tabular}
\end{table}

\begin{table}[h!]
\centering
\caption{Hyperparameters of PPO training for Minigrid Pre-train}
\label{tab:ppo_minigrid_pretrain_hyperparameters}
\begin{tabular}{c|c}
\toprule
\textbf{Hyperparameter} & \textbf{Value} \\
\midrule
Learning Rate & 0.0002 \\
Number of Epochs & 20 \\
Policy Network Architecture & CustomCNN \\
Training Steps & 2e7 \\
Save Frequency & 64000\\
N Steps & 1600\\
\bottomrule
\end{tabular}
\end{table}

\begin{table}[h!]
\centering
\caption{Hyperparameters of PPO training for Minigrid Fine-tune}
\label{tab:ppo_minigrid_finetune_hyperparameters}
\begin{tabular}{c|c}
\toprule
\textbf{Hyperparameter} & \textbf{Value} \\
\midrule
Learning Rate & 2e-5 \\
Policy Network Architecture & CustomCNN \\
Training Steps & 400000 \\
Save Frequency & 16000\\
N Steps & 1600\\
\bottomrule
\end{tabular}
\end{table}

CustomCNN is a customized 2-layer convolutional neural network (CNN) feature extractor designed to process image inputs.

For the BlockedUnlockPickup task, due to its extreme difficulty, the model was unable to learn effectively from scratch. To address this, we used a pre-trained model from the UnlockPickup task and continued pretraining it on the BlockedUnlockPickup task.

Now, regarding the data collection process, for both Dark Room and Minigrid, we selected a series of models and arranged them to reflect their training progression: the earlier models are those that showed continuous performance improvement, while the later ones are those that had reached convergence. This arrangement was made with the intention that the ICRL policy learned from these models would also achieve stability after performance convergence. For each PPO process, we selected 40 models to sample data from a horizon of 400. In Minigrid, PPO checkpoints are saved every 64,000 environment steps during pretraining and every 16,000 steps during fine-tuning. For Dark Room, this process resulted in a total of 40,000 trajectories. For each task in Minigrid, we collected a total of 170,000 trajectories.

\subsection{Details of Implementations}
\label{appendix:imple}

\textbf{AD (Algorithm Distillation):} AD~\citep{AD} is trained on continuously improving trajectories generated from the agent's experience. In this framework, the model learns to predict the next action based on a history of states, actions, and rewards, effectively transforming reinforcement learning into a supervised learning problem.

\textbf{AD$^\epsilon$}: AD$^\epsilon$~\citep{ADeps} builds upon AD by replacing real online trajectories with simulated trajectories. These simulated trajectories are generated by sampling from a noised model, with the noise progressively reduced during training. This approach allows for greater flexibility in training the model by not relying strictly on real data. It also makes the model more robust to imperfections in the trajectory data.

\textbf{IDT (In-Context Decision Transformer):} IDT~\citep{IDT} extends the AD framework by reordering the context according to episode rewards and introducing a hierarchical decision-making structure. This hierarchical approach enables the model to handle longer horizons and more complex tasks. Specifically, IDT organizes the contexts into different decision levels, with the high-level model focusing on broader task goals, while the decision model focuses on more immediate decisions. This method allows for better handling of tasks with long temporal dependencies.

We implement all baselines as well as our method on a GPT-2~\citep{gpt2} based backbone, ensuring comparable parameter scales and closely matched architectural hyperparameters.

Due to the context length limitation of our GPT-2 model, we set the context length to 400, while GPT-2 has a maximum horizon of 1024. This constraint leads to a situation where each position in the Transformer corresponds to a time step's context, rather than having continuous three or four positions in the Minigrid environment (with a 7x7x3 observation) corresponding to one timestep's context. For this, we use 2 convolutional layers followed by 1 MLP layer, embedding the observations into 64-dimensional vectors. Actions are one-hot encoded into a 7-dimensional vector, while rewards and Context Values are represented as single-dimensional values.

For AD$^\epsilon$, the model sequence length is 40. For each position \( i \) in the model sequence, we set \( \epsilon \) as a function: 
$\epsilon_i = \min\left(\frac{i}{30}, 1\right)$,
which ensures consistency with AD and other algorithms in the Minigrid setting. The first 30 models correspond to trajectories with continuously improving performance, while the remaining 10 represent near-stable (optimal) models.

For IDT, we implement the hierarchical structure as outlined in the original paper. The structure includes three models: a decision model, a high-level decision model, and a reviewing decisions model. The high-level decision model's context timestep interval is set to 5, and both the decision model and high-level decision model share the same architecture and hyperparameters. The reviewing decisions model is a 2-layer MLP. The embedding size for the high-level decision model is set to 64.

For CV-ICRL, in Dark Room, we estimate the context-value target using the normalized average episode reward of the source policy. We consider the maximum average episode reward (AER) from the model sequence as 1. In Minigrid, we use the average episode reward of the source policy as the value-token target, as the max AER for these tasks is set to 1. For Minigrid, each of the 40 selected source checkpoints is evaluated in 20 environments for 500 steps, pooling completed episodes to estimate its mean return. This requires 10,000 transitions per checkpoint and 400,000 transitions per task type.

Checkpoint-level Monte Carlo evaluation estimates source-policy performance without an additional learned value estimator; its accuracy depends on return stochasticity and evaluation coverage. GAE and TD($\lambda$) require state-value estimates, while Retrace additionally requires Q/value estimates and behavior/target policy information. Retrofitting an AD dataset requires the source checkpoints and their trajectory mapping. If checkpoint-return estimates were not retained, the checkpoint evaluations must be performed.

\subsection{Additional DPT Baseline}
\label{appendix:dpt_baseline}

Decision-Pretrained Transformer (DPT)~\citep{DPT} is a representative ICRL method based on supervised pretraining and posterior-sampling-style adaptation. To broaden the comparison beyond the AD-like baselines in the main text, we evaluate DPT on our Minigrid benchmark and report both success rate and Average Episode Return (AER) in \Cref{tab:dpt_minigrid_main,tab:dpt_minigrid_ood}.

DPT achieves non-trivial performance on some navigation tasks but is less competitive on tasks requiring more structured sequential behavior or cross-task-type generalization. These results provide an additional reference point for ICRL methods beyond the main AD-like comparison.

\begin{table}[h]
\centering
\caption{DPT comparison on Minigrid task types used in the main experiments. Each entry reports success rate / AER.}
\label{tab:dpt_minigrid_main}
\begin{tabular}{lccc}
\toprule
Environment & AD & CV-ICRL & DPT \\
\midrule
LavaCrossingS9N2 & 100.00\% / 0.9413 & 100.00\% / 0.9445 & 85.00\% / 0.8114 \\
LavaCrossingS9N3 & 100.00\% / 0.9319 & 100.00\% / 0.9241 & 75.00\% / 0.6688 \\
SimpleCrossingS11N5 & 100.00\% / 0.8746 & 100.00\% / 0.9051 & 0.00\% / 0.0000 \\
SimpleCrossingS9N3 & 100.00\% / 0.9308 & 100.00\% / 0.9464 & 5.00\% / 0.0034 \\
\bottomrule
\end{tabular}
\end{table}

\begin{table}[h]
\centering
\caption{DPT comparison on unseen Minigrid task types for cross-task-type generalization. Each entry reports success rate / AER.}
\label{tab:dpt_minigrid_ood}
\begin{tabular}{lccc}
\toprule
Environment & AD & CV-ICRL & DPT \\
\midrule
DoorKey-6x6 & 90.00\% / 0.4473 & 100.00\% / 0.7592 & 0.00\% / 0.0000 \\
DoorKey-8x8 & 100.00\% / 0.4381 & 100.00\% / 0.6302 & 0.00\% / 0.0000 \\
FourRooms & 75.00\% / 0.1076 & 65.00\% / 0.1898 & 25.00\% / 0.2292 \\
LavaCrossingS11N5 & 100.00\% / 0.6373 & 100.00\% / 0.8092 & 0.00\% / 0.0000 \\
\bottomrule
\end{tabular}
\end{table}

\subsection{Additional Continuous-Control Results on MetaWorld}
\label{appendix:metaworld}

To examine whether CV-ICRL also applies beyond discrete-control domains, we conduct a small-scale evaluation on two MetaWorld ML1 tasks. MetaWorld provides continuous-control manipulation environments, which differ substantially from Dark Room, Minigrid, and Procgen. As shown in \Cref{tab:metaworld_results}, CV-ICRL outperforms AD on both tasks in terms of Average Episode Return (AER).

\begin{table}[h]
\centering
\caption{Additional MetaWorld continuous-control results. We report Average Episode Return (AER); higher is better.}
\label{tab:metaworld_results}
\begin{tabular}{lcc}
\toprule
Method & ML1-reach-v3 & ML1-push-v3 \\
\midrule
CV-ICRL & \textbf{426.0} & \textbf{55.2} \\
AD & 253.2 & 18.2 \\
\bottomrule
\end{tabular}
\end{table}

\subsection{Analysis of Confidence Bands}
\label{appendix:confidence_bands}

The confidence bands in the main figures aggregate performance across unseen tasks. A wider band can therefore reflect larger cross-task dispersion rather than unstable adaptation within each task. This distinction is important for CV-ICRL: the method often improves more on easier tasks while remaining comparable to AD on harder tasks, which can increase the standard deviation of task-level means.

To separate cross-task heterogeneity from within-task temporal instability, we compute statistics over the late stage of Dark Room evaluation (steps 3k--4k). As shown in \Cref{tab:darkroom_confidence_decomp}, CV-ICRL has a larger standard deviation across task means, but substantially lower within-task temporal standard deviation. This supports our interpretation that wider aggregate bands in some plots come from task difficulty heterogeneity rather than unstable per-task adaptation.

\begin{table}[h]
\centering
\caption{Decomposition of Dark Room variability in the late evaluation stage (steps 3k--4k). ``Std across tasks'' measures dispersion of task-level mean returns, while within-task std measures temporal fluctuation within each task.}
\label{tab:darkroom_confidence_decomp}
\begin{tabular}{lcccc}
\toprule
Algorithm & \makecell{Mean task\\reward} & \makecell{Std across\\tasks} & \makecell{Mean within-task\\std} & \makecell{Median within-task\\std} \\
\midrule
AD & 24.56 & 23.04 & 7.13 & 6.64 \\
CV-ICRL & $\bm{31.13}$ & $\bm{36.33}$ & $\bm{3.95}$ & $\bm{2.96}$ \\
\bottomrule
\end{tabular}
\end{table}

\subsection{Additional Baseline Comparisons}
\label{appendix:additional_baselines}

Table~\ref{tab:cr_appendix_baselines_full} reports the full comparison. We adapt Reinformer and IC-CQL to the retained cross-episode Minigrid setting. Reinformer represents max-return sequence modeling, while IC-CQL uses history-conditioned twin Q-functions. Both baselines use fixed checkpoint rules across the four task types, without task-specific selection. IC-CQL selects actions greedily from Q-values, whereas the sequence policies use categorical sampling. We report results over 20 unseen tasks per task type and an 8,000-step interaction budget.

Both adaptations use the original Minigrid trajectory corpus, the Minigrid CNN encoder, and a causal Transformer with four layers, four attention heads, embedding dimension 256, dropout 0, and a 400-step context window. Each model is trained for 15 epochs with effective batch size 32 and training seed 0. Reinformer uses AdamW with learning rate $10^{-4}$, weight decay $10^{-4}$, 5,000 warmup updates, gradient clipping at 0.25, expectile $\tau=0.99$, and return-loss weight 1. Its episode-local RTG targets use $\gamma=1$; incomplete episode suffixes are masked. A fixed training-corpus RTG scale is used throughout training and validation. The predicted RTG and the same causal hidden state feed a linear action head, requiring one Transformer pass per action. Its 1,799-parameter action head matches CV-ICRL's action head.

IC-CQL uses twin Q-functions with the TD target $r_t+\gamma(1-d_t)\min_{j\in\{1,2\}}\max_a Q_j^{\mathrm{target}}(h_{t+1},a)$, where $h_{t+1}$ is the history-conditioned representation and $d_t$ indicates termination. Its conservative term is the sum of the two action cross-entropies with label smoothing 0.3 and weight 0.01. Optimization uses Adam with learning rate $3\times10^{-4}$, discount $\gamma=0.9$, target-update coefficient 0.005, and gradient clipping at 1. The reported checkpoints are fixed at epoch 10 for Reinformer and epoch 5 for IC-CQL on every task type. IC-CQL acts using $\arg\max_a Q_1(h_t,a)$. Both added baselines use evaluator seeds 20000--20019, whereas the CV-ICRL reference uses seeds 0--19; comparisons are therefore not seed-paired. Reported means weight the 20 tasks equally, and standard deviations are population standard deviations across tasks.

\begin{table}[htbp]
\centering
\small
\setlength{\tabcolsep}{4pt}
\caption{Additional baseline results over 20 unseen tasks per task type. AER and LER are higher-is-better; IF (\%) is lower-is-better. Entries show mean $\pm$ population standard deviation. Best mean values are \textbf{bolded}.}
\label{tab:cr_appendix_baselines_full}
\begin{tabular}{lcccc}
\toprule
Task type & Metric & CV-ICRL & IC-CQL & Reinformer \\
\midrule
SimpleCrossingS9N3 & AER & $\bm{0.945\pm0.015}$ & $0.883\pm0.014$ & $0.886\pm0.020$ \\
 & LER & $\bm{0.949\pm0.012}$ & $0.860\pm0.287$ & $0.801\pm0.327$ \\
 & IF (\%) & $\bm{1.561\pm1.880}$ & $8.575\pm1.381$ & $10.735\pm2.834$ \\
\midrule
SimpleCrossingS11N5 & AER & $\bm{0.902\pm0.076}$ & $0.303\pm0.192$ & $0.864\pm0.225$ \\
 & LER & $\bm{0.921\pm0.066}$ & $0.048\pm0.210$ & $0.895\pm0.208$ \\
 & IF (\%) & $13.209\pm13.108$ & $68.524\pm17.792$ & $\bm{11.711\pm22.245}$ \\
\midrule
LavaCrossingS9N3 & AER & $\bm{0.921\pm0.051}$ & $0.616\pm0.420$ & $0.094\pm0.160$ \\
 & LER & $\bm{0.939\pm0.026}$ & $0.572\pm0.467$ & $0.048\pm0.208$ \\
 & IF (\%) & $\bm{5.228\pm8.787}$ & $35.376\pm43.752$ & $88.921\pm16.946$ \\
\midrule
BlockedUnlockPickup & AER & $\bm{0.955\pm0.010}$ & $0.195\pm0.262$ & $0.942\pm0.021$ \\
 & LER & $\bm{0.961\pm0.005}$ & $0.048\pm0.209$ & $0.934\pm0.046$ \\
 & IF (\%) & $\bm{1.825\pm2.732}$ & $76.754\pm25.521$ & $7.078\pm6.987$ \\
\bottomrule
\end{tabular}
\end{table}

\subsection{Alternative Value Supervision}
\label{appendix:alternative_supervision}

To examine state-dependent supervision, we use the critic associated with the source PPO checkpoint:
\[
y_{\mathrm{PPO}}(s_t,C_t)=V^{\pi_{k(C_t)}}(s_t).
\]
Here, $V^{\pi_{k(C_t)}}$ denotes the source critic estimate. The context identifies the teacher checkpoint, and the critic takes the state as input, making target construction context-indexed and state-dependent.

As shown in Table~\ref{tab:cr_appendix_ppo_state}, PPO-state supervision has higher AER on both SimpleCrossing task types, while checkpoint-return supervision has higher AER on the other four. Neither target is uniformly superior.

Evaluation uses 20 unseen tasks per task type, 8,000 interaction steps per task, categorical action sampling, and context retention across episode resets. These evaluations were run on single RTX 3090 or RTX 4090 GPUs (24\,GB). AER is computed by averaging episode returns within each task and then weighting the 20 tasks equally.

\begin{table}[htbp]
\centering
\small
\setlength{\tabcolsep}{4pt}
\caption{Context-indexed PPO-state supervision. AER is reported as mean $\pm$ population standard deviation over 20 unseen tasks. The context selects the source PPO checkpoint, and its critic provides a state-dependent target.}
\label{tab:cr_appendix_ppo_state}
\begin{tabular}{lcc}
\toprule
Task type & CV-ICRL & PPO-state target \\
\midrule
SimpleCrossingS9N3 & $0.945\pm0.015$ & $0.949\pm0.001$ \\
SimpleCrossingS11N5 & $0.902\pm0.076$ & $0.916\pm0.067$ \\
BlockedUnlockPickup & $0.955\pm0.010$ & $0.922\pm0.041$ \\
Unlock & $0.966\pm0.008$ & $0.948\pm0.015$ \\
LavaCrossingS9N3 & $0.921\pm0.051$ & $0.914\pm0.011$ \\
LavaCrossingS9N2 & $0.943\pm0.015$ & $0.933\pm0.052$ \\
\bottomrule
\end{tabular}
\end{table}

\subsection{Deployment Value-Token Interventions}
\label{appendix:token_interventions}

We intervene on value-token feedback during deployment of the same pretrained CV-ICRL policy. For each evaluation stream, let $\hat v_t$ be the raw critic prediction. The Minigrid evaluator uses the running-maximum write-back rule
\[
m_{-1}=0,\qquad m_t=\max(m_{t-1},\hat v_t).
\]
The normal condition writes $m_t$ into subsequent context; zero writes 0, and frozen holds all feedback tokens at the stream's first write-back value $m_0=\max(0,\hat v_0)$. Each condition runs for 8,000 interaction steps across multiple episodes, retaining context and write-back state across episode resets.

The AER drops under zero and frozen feedback in Table~\ref{tab:cr_appendix_token_interventions} support behavioral reliance on token content and continued write-back. These experiments characterize the tested deployment behavior; they do not measure value calibration or robustness to longer per-episode horizons.

\begin{table}[htbp]
\centering
\small
\setlength{\tabcolsep}{4pt}
\caption{Value-token interventions over 8,000 interaction steps across multiple episodes. AER is reported as mean $\pm$ population standard deviation across 20 evaluator streams. These results test behavioral reliance on token content and continued write-back.}
\label{tab:cr_appendix_token_interventions}
\begin{tabular}{lcc}
\toprule
Condition & SimpleCrossingS9N3 & BlockedUnlockPickup \\
\midrule
Normal & $0.942\pm0.020$ & $0.936\pm0.027$ \\
Zero token & $0.090\pm0.152$ & $0.097\pm0.124$ \\
Frozen token & $0.286\pm0.296$ & $0.241\pm0.219$ \\
\bottomrule
\end{tabular}
\end{table}

\subsection{Effect of Context Length}
\label{appendix:context_length}

Table~\ref{tab:cr_appendix_context_length} compares the $H=400$ main-table reference with a separately trained $H=200$ model, using equally weighted AER across 20 evaluation tasks. The shorter-context model has lower AER on both LavaCrossingS9N3 and SimpleCrossingS9N3, with a larger difference on SimpleCrossingS9N3.

Context is retained across episode resets as a FIFO window. Defining $C_t$ as the visible history preserves the conditional-value formulation, while information discarded outside the window is unavailable for later adaptation.

\begin{table}[htbp]
\centering
\small
\setlength{\tabcolsep}{4pt}
\caption{Context-length comparison over 20 unseen tasks per task type. The $H=400$ column uses the main-table reference; $H=200$ uses a separately trained model. AER is reported as the equally weighted task mean $\pm$ population standard deviation.}
\label{tab:cr_appendix_context_length}
\begin{tabular}{lcc}
\toprule
Task type & $H=400$ (main-table reference) & $H=200$ \\
\midrule
LavaCrossingS9N3 & $0.921\pm0.051$ & $0.832\pm0.282$ \\
SimpleCrossingS9N3 & $0.945\pm0.015$ & $0.037\pm0.053$ \\
\bottomrule
\end{tabular}
\end{table}

\subsection{Compute Resources}
\label{appendix:compute}
The original experiments are run on a single NVIDIA RTX 3090 GPU (24\,GB VRAM). Training proceeds in two stages: \emph{Stage~1} runs PPO on each task type to produce a curriculum of policy checkpoints and the corresponding trajectory dataset; \emph{Stage~2} supervised-pretrains each ICRL method (AD, AD$^\epsilon$, IDT, and CV-ICRL) on the resulting dataset. Approximate per-run wall-clock times are summarized in Table~\ref{tab:compute_resources}. The added value-supervision evaluations also used RTX 4090 GPUs, as documented with their experimental protocol.

\begin{table}[h]
\caption{Approximate per-run compute for the original experiments on a single RTX 3090 (24\,GB). Stage~1 is PPO pretraining and dataset generation; Stage~2 is supervised pretraining of an ICRL method. ``per task type'' multiplies over the number of task types in the environment; ``per method'' multiplies over AD, AD$^\epsilon$, IDT, and CV-ICRL. Added baseline, value-supervision, and deployment-intervention experiments are excluded.}
\label{tab:compute_resources}
\centering
\begin{tabular}{lll}
\toprule
Environment & Stage~1 (PPO) & Stage~2 (ICRL) \\
\midrule
Dark Room & 4\,h (total) & 4\,h (total) \\
Minigrid (6 task types) & 24\,h per task type & 8\,h per (method, task type) \\
Procgen (3 task types) & 16\,h per task type & 10\,h per (method, task type) \\
\bottomrule
\end{tabular}
\end{table}

For the original experiments across all three environments, the total training compute is approximately $8 + 6\times 24 + 6\times 4\times 8 + 3\times 16 + 3\times 4\times 10 = 512$ GPU-hours. This estimate excludes added baseline and value-supervision experiments, deployment interventions, hyperparameter search, and earlier prototyping.

\clearpage
\section*{NeurIPS Paper Checklist}

\begin{enumerate}

\item {\bf Claims}
    \item[] Question: Do the main claims made in the abstract and introduction accurately reflect the paper's contributions and scope?
    \item[] Answer: \answerYes{}
    \item[] Justification: The abstract and the Introduction state our main contributions: a policy-gradient-like interpretation of ICRL that relates instability to high-variance return-based credit signals; the CV-ICRL framework, which predicts a context value and writes it back into the context to enable TD-style updates; and empirical gains in stability and returns on Dark Room, Minigrid, and Procgen, with additional out-of-distribution evidence. These claims are supported by the theoretical analysis in Sections~3 and~4 and the experiments in Section~5.

\item {\bf Limitations}
    \item[] Question: Does the paper discuss the limitations of the work performed by the authors?
    \item[] Answer: \answerYes{}
    \item[] Justification: The Conclusion includes an explicit limitations paragraph discussing the scope of the policy-gradient interpretation, the current scope of cross-task-type and continuous-control evaluations, and possible gains from more refined context-value estimation.

\item {\bf Theory assumptions and proofs}
    \item[] Question: For each theoretical result, does the paper provide the full set of assumptions and a complete (and correct) proof?
    \item[] Answer: \answerYes{}
    \item[] Justification: Property~\ref{prop:rb_var_reduction} in Section~4 states the assumptions for variance reduction using the true conditional value, with a complete proof in Appendix~\ref{app:rb_proof}. The Improved Performance Bound and the sufficient value-error condition in Corollary~\ref{cor:value_error} are stated and proved in Appendix~\ref{Appendix:improved_performance_bound}.

    \item {\bf Experimental result reproducibility}
    \item[] Question: Does the paper fully disclose all the information needed to reproduce the main experimental results of the paper to the extent that it affects the main claims and/or conclusions of the paper (regardless of whether the code and data are provided or not)?
    \item[] Answer: \answerYes{}
    \item[] Justification: Section~4 describes the supervised pretraining procedure and loss; Section~5 specifies the benchmarks, baselines, and evaluation protocol; Appendices~\ref{Appendix:env}--\ref{appendix:imple} provide environment details, dataset preparation, and per-method architectural and training hyperparameters. Appendices~\ref{appendix:additional_baselines}--\ref{appendix:context_length} document the added baseline comparisons, alternative value supervision, deployment interventions, and context-length comparison. The code repository and training-data access instructions are linked in the abstract for further reproducibility.

\item {\bf Open access to data and code}
    \item[] Question: Does the paper provide open access to the data and code, with sufficient instructions to faithfully reproduce the main experimental results, as described in supplemental material?
    \item[] Answer: \answerYes{}
    \item[] Justification: The repository linked in the abstract (\url{https://github.com/Bluixe/CV-ICRL}) provides training, evaluation, and data-collection code, and download instructions for six Minigrid and three Procgen CV-ICRL training files.

\item {\bf Experimental setting/details}
    \item[] Question: Does the paper specify all the training and test details (e.g., data splits, hyperparameters, how they were chosen, type of optimizer) necessary to understand the results?
    \item[] Answer: \answerYes{}
    \item[] Justification: Section~5 describes the benchmarks, AER and LER metrics, Instability Frequency, and test-task evaluation protocol; Appendix~\ref{appendix:imple} reports the backbone configuration, context length, and observation/action embeddings. Appendices~\ref{appendix:additional_baselines}--\ref{appendix:context_length} specify the added experiments' training and evaluation settings, including fixed checkpoint rules, action selection, interaction budgets, token interventions, and context retention. The added baseline comparisons use different evaluator seed sets and are not seed-paired.

\item {\bf Experiment statistical significance}
    \item[] Question: Does the paper report error bars suitably and correctly defined or other appropriate information about the statistical significance of the experiments?
    \item[] Answer: \answerYes{}
    \item[] Justification: The Minigrid main table and added appendix tables report equally weighted means $\pm$ population standard deviations across 20 unseen tasks (or evaluator streams for the token interventions). These standard deviations describe variation across evaluation tasks, not variation across independent training seeds. Main figures (e.g., Figure~\ref{fig:main_result}) report task-mean returns with shaded 95\% confidence intervals, as stated in their captions; Appendix~\ref{appendix:confidence_bands} distinguishes cross-task dispersion from within-task temporal instability.

\item {\bf Experiments compute resources}
    \item[] Question: For each experiment, does the paper provide sufficient information on the computer resources (type of compute workers, memory, time of execution) needed to reproduce the experiments?
    \item[] Answer: \answerYes{}
    \item[] Justification: Appendix~\ref{appendix:compute} provides GPU specifications, approximate training times, and an estimate of training compute. Additional experimental settings are described in the corresponding appendix sections.

\item {\bf Code of ethics}
    \item[] Question: Does the research conducted in the paper conform, in every respect, with the NeurIPS Code of Ethics \url{https://neurips.cc/public/EthicsGuidelines}?
    \item[] Answer: \answerYes{}
    \item[] Justification: The research is purely methodological reinforcement learning conducted in simulated, open-source environments (Dark Room, Minigrid, Procgen); it does not involve human subjects, personal or sensitive data, or deployments with foreseeable harm, and conforms with the NeurIPS Code of Ethics.

\item {\bf Broader impacts}
    \item[] Question: Does the paper discuss both potential positive societal impacts and negative societal impacts of the work performed?
    \item[] Answer: \answerYes{}
    \item[] Justification: The Impact Statement section discusses the scope and societal implications of the work; as foundational research on test-time adaptation in simulated RL, the work has limited direct path to negative applications, which is noted there.

\item {\bf Safeguards}
    \item[] Question: Does the paper describe safeguards that have been put in place for responsible release of data or models that have a high risk for misuse (e.g., pre-trained language models, image generators, or scraped datasets)?
    \item[] Answer: \answerNA{}
    \item[] Justification: The released code and training histories concern adaptation in simulated grid-world and procedural environments, which do not pose a high risk of misuse.

\item {\bf Licenses for existing assets}
    \item[] Question: Are the creators or original owners of assets (e.g., code, data, models), used in the paper, properly credited and are the license and terms of use explicitly mentioned and properly respected?
    \item[] Answer: \answerYes{}
    \item[] Justification: All third-party assets used in the paper, including Minigrid~\citep{MinigridMiniworld23} (Apache 2.0), the Procgen benchmark, the GPT-2 architecture~\citep{gpt2} (MIT), and the PPO algorithm~\citep{ppo}, are cited in the corresponding sections, and we use them in accordance with their licenses.

\item {\bf New assets}
    \item[] Question: Are new assets introduced in the paper well documented and is the documentation provided alongside the assets?
    \item[] Answer: \answerYes{}
    \item[] Justification: The code repository linked in the abstract (\url{https://github.com/Bluixe/CV-ICRL}) includes a README documenting the environment setup, training scripts, and evaluation procedure for the main experiments. The Minigrid and Procgen CV-ICRL training histories are published at \url{https://huggingface.co/datasets/bluixe/CV-ICRL}, with the fixed release tag \texttt{camera-ready-v1}, under CC BY 4.0. The dataset card and manifest document the data formats, labels, and release scope, and a versioned downloader is provided. Protocols for the added experiments are described in Appendices~\ref{appendix:additional_baselines}--\ref{appendix:context_length}.

\item {\bf Crowdsourcing and research with human subjects}
    \item[] Question: For crowdsourcing experiments and research with human subjects, does the paper include the full text of instructions given to participants and screenshots, if applicable, as well as details about compensation (if any)?
    \item[] Answer: \answerNA{}
    \item[] Justification: The paper does not involve crowdsourcing or research with human subjects; all experiments are conducted in simulated environments.

\item {\bf Institutional review board (IRB) approvals or equivalent for research with human subjects}
    \item[] Question: Does the paper describe potential risks incurred by study participants, whether such risks were disclosed to the subjects, and whether Institutional Review Board (IRB) approvals (or an equivalent approval/review based on the requirements of your country or institution) were obtained?
    \item[] Answer: \answerNA{}
    \item[] Justification: The paper does not involve human subjects; no IRB approval is required.

\item {\bf Declaration of LLM usage}
    \item[] Question: Does the paper describe the usage of LLMs if it is an important, original, or non-standard component of the core methods in this research? Note that if the LLM is used only for writing, editing, or formatting purposes and does \emph{not} impact the core methodology, scientific rigor, or originality of the research, declaration is not required.
    \item[] Answer: \answerNA{}
    \item[] Justification: LLMs were used only for writing and editing assistance.

\end{enumerate}

\end{document}